\documentclass[nojss]{jss}

\usepackage{thumbpdf,lmodern}

\usepackage{amsmath}
\usepackage{amsfonts}
\usepackage{amssymb}
\usepackage{dsfont}
\usepackage[T1]{fontenc}
\usepackage{mathtools}  
\usepackage[plain,noend,ruled]{algorithm2e}
\usepackage[english]{babel}
\usepackage{natbib}
\usepackage{xcolor}
\usepackage{graphicx}
\usepackage{caption}
\usepackage{soul}
\usepackage{amsthm}

\newcommand{\cN}{\mathcal{N}}

\newcommand{\cS}{\mathcal{S}}
\newcommand{\cG}{\mathcal{G}}
\newcommand{\cM}{\mathcal{M}}
\newcommand{\R}{\mathbb{R}}
\newcommand{\N}{\mathbb{N}}
\newcommand{\bX}{\mathbf{X}}

\newcommand{\bt}{\mathbf{t}}
\newcommand{\by}{\mathbf{y}}
\newcommand{\ba}{\mathbf{a}}
\newcommand{\bb}{\mathbf{b}}
\newcommand{\bx}{\mathbf{x}}
\newcommand{\bw}{\mathbf{w}}
\newcommand{\bu}{\mathbf{u}}
\newcommand{\be}{\mathbf{e}}
\newcommand{\bB}{\mathbf{B}}
\newcommand{\tr}[1]{#1}

\newtheorem{theorem}{Theorem}
\newtheorem{lemma}[theorem]{Lemma}

\DeclareMathOperator*{\argmin}{arg\,min}
\DeclareMathOperator*{\diag}{diag}
\DeclareMathOperator*{\rank}{rank}

\DeclareMathOperator*{\minimize}{minimize\ \ }
\DeclareMathOperator*{\subjectto}{subject\ to\ \ }

\title{Modeling longitudinal data using matrix completion}
\Plaintitle{Modeling longitudinal data using matrix completion}
\Shorttitle{Modeling longitudinal data using matrix completion}

\author{{\L ukasz} {Kidzi\'nski}\\Stanford University
\And Trevor Hastie\\Stanford University}
\Plainauthor{{\L ukasz} {Kidzi\'nski}, Trevor Hastie}

\Abstract{
In clinical practice and biomedical research, measurements are often collected sparsely and irregularly in time while the data acquisition is expensive and inconvenient. Examples include measurements of spine bone mineral density, cancer growth through mammography or biopsy, a progression of defective vision, or assessment of gait in patients with neurological disorders. Since the data collection is often costly and inconvenient, estimation of progression from sparse observations is of great interest for practitioners.

  From the statistical standpoint, such data is often analyzed in the context of a mixed-effect model where time is treated as both \tr{a fixed-effect (population progression curve) and a random-effect (individual variability)}. Alternatively, researchers analyze Gaussian processes or functional data where observations are assumed to be drawn from a certain distribution of processes. These models are flexible but rely on probabilistic assumptions, require very careful implementation, specific to the given problem, and tend to be slow in practice.
  
  In this study, we propose an alternative elementary framework for analyzing longitudinal data, relying on matrix completion. Our method yields estimates of progression curves by iterative application of the Singular Value Decomposition. Our framework covers multivariate longitudinal data, regression, and can be easily extended to other settings. As it relies on existing tools for matrix algebra it is efficient and easy to implement. 

  We apply our methods to understand trends of progression of motor impairment in children with Cerebral Palsy. Our model approximates individual progression curves and explains 30\% of the variability. Low-rank representation of progression trends enables identification of different progression trends in subtypes of Cerebral Palsy.
}

\Keywords{Matrix factorization, Matrix completion, Multivariate longitudinal data, Regression, Interpolation}

\Address{
  {\L ukasz} {Kidzi\'nski}\\
  Department of Bioengineering\\
  Stanford University\\
  S.331, James Clark Center,\\
  318 Campus Drive,\\
  Stanford, CA 94305, USA\\
  E-mail: \email{lukasz.kidzinski@stanford.edu}\\
  URL: \url{http://kidzinski.com/}
}

\begin{document}

\section{Motivation}\label{s:motivation}

The key question in medical practice and clinical research is how diseases progress in individual patients. Accurate continuous monitoring of patient's condition could considerably improve prevention and treatment. Many medical tests, such as X-ray, \tr{magnetic resonance imaging (MRI)}, motion capture gait analysis, biopsy, or blood tests, cannot be taken routinely, since they are costly, harmful or inconvenient. Therefore, practitioners and researchers need to reach out for statistical tools to analyze progression of patient's condition with only sparse and noisy observations at hand.



For illustration, consider longitudinal measurements of the gait deviation index (GDI) which is a holistic measure of motor impairment in children. GDI is measured using advanced motion capture hardware and software (Figure \ref{fig:motivation}). Due to high costs, such measurements are taken only a few times in patient's life. By looking at individual processes and by modeling between-subject similarities, we can model the individual progressions despite having access to only a few observations. Such modeling can yield personalized predictions, clustering of patients based on progression patterns, or latent representation of progression patterns, which then could be used as covariates in other predictive models. 

\begin{figure}[h]
 \includegraphics[width=0.49\linewidth]{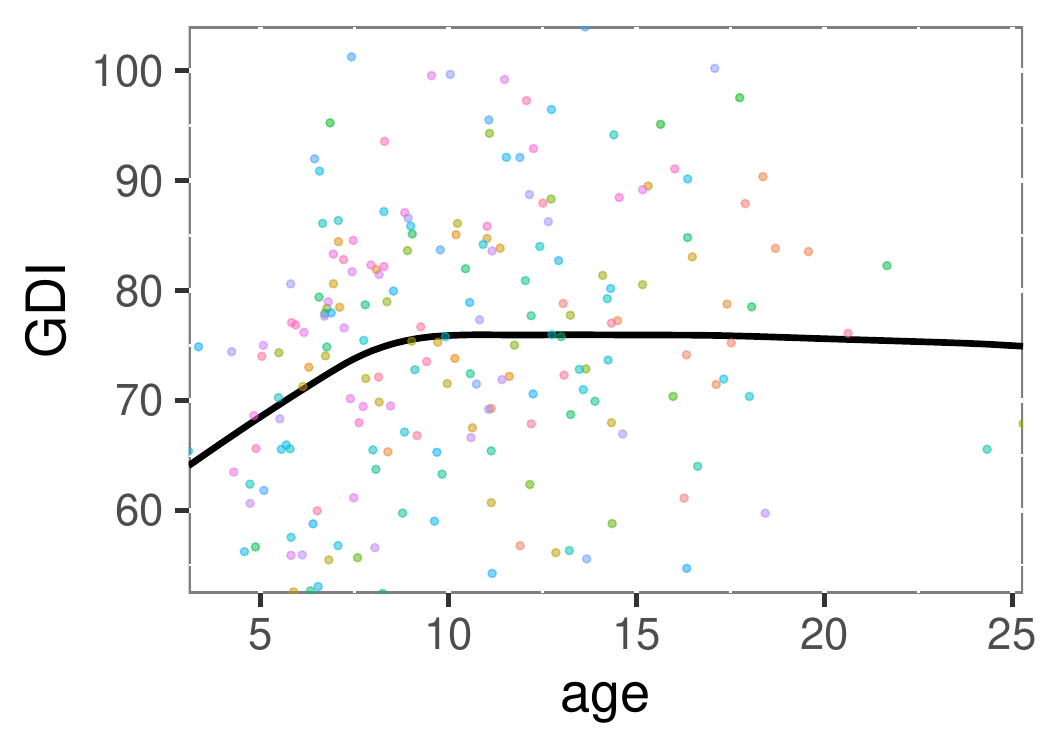}
 \includegraphics[width=0.49\linewidth]{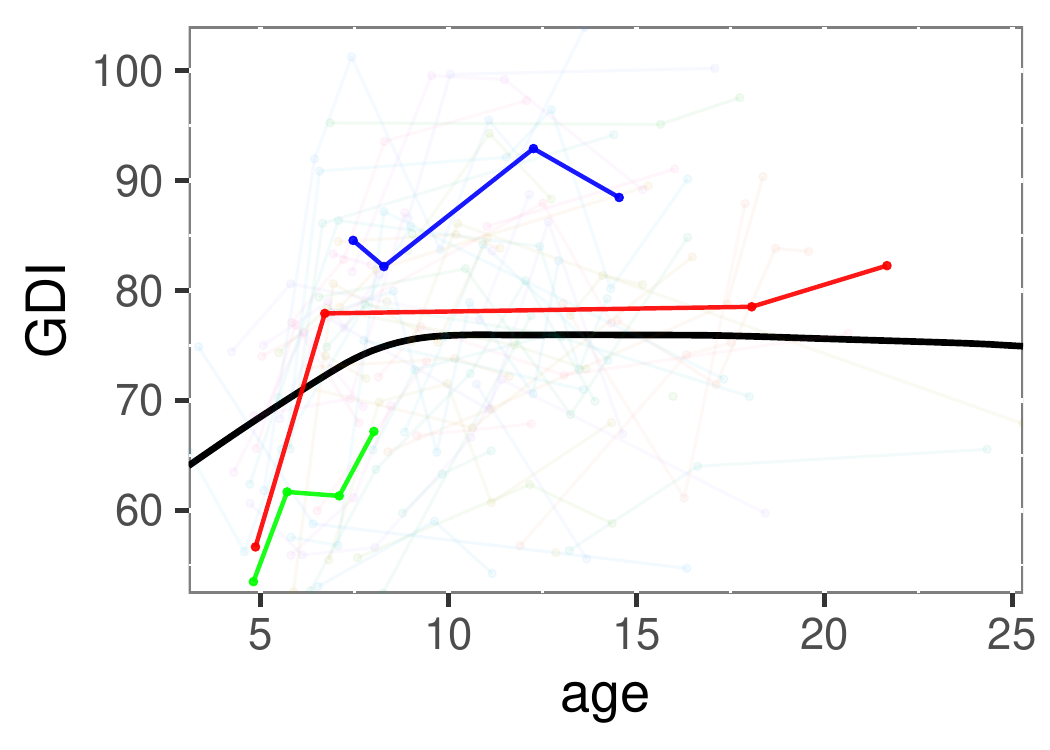}
 \caption{Left plot: We observe the gait deviation index (GDI), a holistic metric of motor impairment, at every visit in a clinic and we derive the population progression (the thick solid curve) using a locally weighted scatterplot smoothing method. Right plot: The analysis of individual patients (connected dots in the right plot in different colors) reveals patterns in individual trends. We highlight 3 randomly selected subjects in red, green, and blue.}
 \label{fig:motivation}
\end{figure}

These kinds of data have been commonly analyzed using linear mixed models, where we treat time as a random effect and nonlinearity of this effect is imposed by the choice of the functional basis \citep{zeger1988models, verbeke1997linear, mcculloch2001generalized}. When data is very sparse, additional constraints on the covariance structure of trajectories are inferred using cross-validation or information criteria \citep{rice2001nonparametric,bigelow2009bayesian}. To further reduce the dimension, practitioners model the covariance as a low-rank matrix \citep{james2000principal,berkey1983longitudinal, yan2017dynamic, hall2006properties, besse1986principal, yao2006penalized, greven2011longitudinal}. Many models have been developed for incorporating additional covariates \citep{song2002semiparametric, liu2009joint, rizopoulos2014combining}. While these methods are used in practice, they tend to be slow and require fine-tuning for each clinical application. Moreover, the probabilistic formulation of the model and dependence on the underlying distributions might hinder applicability or adaption to other practical problems.

In this work, we propose a flexible and efficient framework, using matrix factorization techniques. We focus on the simplicity of the formulation and we implement software which is fast and easy to use.

The manuscript is organized as follows. In Section \ref{s:background}, we introduce the notation of the functional representation of the problem and we discuss optimization techniques. Next, in Section \ref{s:context}, we introduce matrix representation of the problem and show how it maps to the functional representation. We develop our main results and conclude with its links to other results. In Section \ref{s:multivariate}, we extend our framework to the multivariate case. We show example uses of our methods in a simulation study (Section \ref{s:simulation}) and in a data study (Section \ref{s:data-study}). We conclude with a discussion section (Section \ref{s:discussion}).

\section{Background and related work} \label{s:background}


To focus our attention and aid further discussion, we start by introducing notation and methodology for the univariate case. Let $N$ denote the number of subjects. For each individual $i \in \{ 1,2,...,N \}$, we take measurements at $n_i$ irregularly sampled time-points $\mathbf{t}_i = [t_{i,1},t_{i,2},...,t_{i,n_i}]'$. We assume that $t_{min} < t_{i,1} < t_{i,2} < ... < t_{i,n_i} < t_{max}$ for each $i$ and some $t_{min}, t_{max} \in \mathds{R}$. \tr{In this work, we ignore the stochastic nature of sampling timepoints and we assume that points $\bt_i$ are fixed for each individual.} Vectors $\by_i = [y_{i,1,},...,y_{i,n_i}]'$ denote observations corresponding to $\bt_i$ for each individual $i$.

To model individual trajectories given pairs $(\mathbf{t}_i,\mathbf{y}_i)$ practitioners map observations into a low-dimensional space which represents progression patterns. Conceptually, a small distance between individuals in the latent space reflects similar progression patterns.

In this section, we discuss state-of-the-art approaches to estimating this low-dimensional latent embedding. We classify them into three broad categories: the direct approach, mixed-effect models, and low-rank approximations. 

\subsection{Direct approach} \label{ss:direct}

If the set of observed values for each individual is dense, elementary interpolation using a continuous basis can be sufficient for approximating the entire trajectory. Let $\{b_i: i \in \N \}$ be a basis of $L_2([t_{\min},t_{\max}])$, i.e. the space of functions $f : [t_{\min},t_{\max}] \rightarrow \mathds{R}$ such that the integral of $f^2$ on $[t_{\min},t_{\max}]$ is finite, such as splines or polynomials. In practice, we truncate the basis to a finite set of the first $K \in \N$ basis elements. Let $\bb(t) = [b_1(t),b_2(t),...,b_K(t)]'$ be a vector of $K$ basis elements evaluated at a timepoint $t \in (t_{\min},t_{\max})$. Throughout this article we use the word \emph{basis} to refer to some truncated basis of $K$ elements.

To find an individual trajectory, for each subject $i \in \{ 1,...,N \}$, we might use least squares and estimate a set of coefficients $\bw_i \in \R^K$, minimizing squared \tr{Euclidean} distance to observed points
\begin{align}\label{eq:direct-individual}
 \argmin_{\bw_i}\sum_{j=1}^{n_i}\left|y_{i,j} - \bw_i'\bb(t_{i,j})\right|^2.
\end{align}

\tr{This direct approach has two main drawbacks. First, it ignores correlations within the curves, which could potentially improve the fit. Second, if the number of observations $n_i$ for an individual $i$ is smaller or equal to the size of the basis $K$, we can fit a curve with no error leading to overfitting and unreliable estimator of the variance.}

\tr{To remedy the first issue, it is common to estimate covariance operator, compute principal functional components across all individuals and represent curves in the space spanned by the first few of them.} Such representation, referred to as a {\it Karhunen-Lo\`eve} expansion \citep{watanabe1965karhunen,kosambi2016statistics}, has became a foundation of many functional data analysis workflows \citep{ramsay1991some,yao2005linear,cnaan1997tutorial,laird1988missing,horvath2012inference,besse1997simultaneous}. 

A basic idea to remedy the second issue is to estimate both the basis coefficients and the covariance structure simultaneously. Linear mixed-effect models provide a convenient solution to this problem.

\subsection{Linear mixed-effect models} \label{ss:lmm}

A common approach to modeling longitudinal observation $(\mathbf{t}_i, \mathbf{y}_i)$ is to assume that data come from a linear mixed-effect model (LMM) \citep{verbeke1997linear, zeger1988models}. We operate in a functional space with a basis $\bb(t)$ of $K$ elements and we assume there exists a \emph{fixed effect} $\mu(t) = m' \bb(t)$, where $m = [m_1,...,m_K]'$ for $m_i \in \R$ for $1 \leq i \leq K$. We model the individual \emph{random effect} as a vector of basis coefficients. In the simplest form, we assume
\begin{align}\label{eq:latent-probabilistic}
 \mathbf{w}_i \sim \cN(0, \Sigma),
\end{align}
where $\Sigma$ is a $K \times K$ covariance matrix. We model individual observations as
\begin{align}\label{eq:probabilistic}
 \mathbf{y}_i|\mathbf{w}_i \sim \cN(\mu_i + B_i\mathbf{w}_i, \sigma^2I_{n_i}),
\end{align}
where $\mu_i = [\mu(t_{i,1}),\mu(t_{i,2}),...,\mu(t_{i,n_i})]'$, $\sigma$ is the standard deviation of observation error and $B_i = [\bb(t_{i,1}),...,\bb(t_{i,n_i})]'$ is the basis evaluated at timepoints defined in $\bt_i$. Estimation of the model parameters is typically accomplished by the expectation-maximization (EM) algorithm \citep{laird1982random}. For \tr{predicting} coefficients $\bw_i$ one can use the best unbiased linear predictor (BLUP) \citep{henderson1950estimation,robinson1991blup}. 

Since the LMM estimates the covariance structure and the individual fit simultaneously, it reduces the problem of overfitting of $\bw_i$, present in the direct approach. However, this model is only applicable if we observe a relatively large number of observations per subject since we attempt to estimate $K$ coefficients for every subject.

To model trajectories from a small number of observations, practitioners further constrain the covariance structure. If we knew the functions which contribute the most to the random effect, we could fit an LMM in a smaller space spanned by these functions. We explore possibilities to learn the basis from the data using low-rank approximations of the covariance matrix. 

\subsection{Low-rank approximations} \label{ss:reduced-rank}

There are multiple ways to constrain the covariance structure. We can use cross-validation or information criteria to choose the best basis, the number of elements or positions of spline knots \citep{rice2001nonparametric,bigelow2009bayesian}. Alternatively, we can place a prior on the covariance matrix \citep{maclehose2009nonparametric}.

Another solution is to restrict the latent space to $q < K$ dimensions and learn from the data a mapping $A \in \R^{K \times q}$ between the latent space and the basis. In the simplest scenario with Gaussian noise, observations can be modeled as
\begin{align}\label{eq:james-model}
 \mathbf{y}_i | \bw_i \sim \cN(\mu_i + B_i A \mathbf{w}_i, \sigma^2I_{n_i}),
\end{align}
following the notation from \eqref{eq:probabilistic}.

\citet{james2000principal} propose an EM algorithm for finding model parameters and predicting latent variables $\bw_i \in \R^q$ in \eqref{eq:james-model}. In the expectation stage, they compute the conditional mean of $\bw_i$ given $y_i$ and the model parameters, while 
in the maximization stage, with $\bw_i$ assumed observed, they maximize the likelihood with respect to $\{\mu,A,\sigma\}$. The likelihood, given $\bw_i$, takes the form
\begin{align*}
\prod_{i=1}^N \frac{1}{(2\pi)^{n_i/2} \sigma^{n_i} |\Sigma|^{1/2}} \exp\{ &-(\by_i - \mu_i - B_i A \bw_i)'(\by_i - \mu_i - B_i A \bw_i) / 2\sigma^2 \nonumber\\
&- \frac{1}{2}\bw_i' \Sigma^{-1} \bw \}.\label{eq:likelihood}
\end{align*}

Another approach to estimating parameters of \eqref{eq:james-model} is to optimize over $\bw_i$ and marginalize $A$ \citep{lawrence2004gaussian}. This approach allows modifying the distance measure in the latent space, using the \emph{kernel trick} \citep{schulam2016disease}.

Estimation of the covariance structure of processes is central to the estimation of individual trajectories. \citet{descary2016functional} propose a method where the estimate of the covariance matrix is obtained through matrix completion.

Methods based on low-rank approximations are widely adopted and applied in practice \citep{berkey1983longitudinal, yan2017dynamic, hall2006properties, besse1986principal, yao2006penalized, greven2011longitudinal}. However, due to their probabilistic formulation and reliance on the distribution assumptions, these models usually need to be carefully fine-tuned for specific situations and existing implementations tend to be slow. This shortcoming motivates us to develop an elementary optimization framework, using existing, extensively studied, and well--optimized tools for matrix algebra.

To illustrate this problem in an elementary setting, let us assume that the data is drawn from two distributions with equal probability taking form \eqref{eq:james-model}, but with two different means: $\mu$ and $-\mu$. Then, the estimated fixed effect will be close to $0$ and the zero-mean prior distribution on the latent variables will draw these variables towards $0$.

\section{Modeling sparse longitudinal processes using matrix completion} \label{s:context}



We pose the problem of trajectory prediction as a matrix completion problem, and we solve it using sparse matrix factorization techniques \citep{rennie2005fast, candes2009exact,fithian2018flexible}. In the classical matrix completion problem, the objective is to predict elements of a sparsely observed matrix using its known elements while minimizing a specific criterion, commonly Mean Squared Error (MSE). The motivating example is the ``Netflix Prize'' competition \citep{bennett2007netflix}, where teams were tasked to predict unknown movie ratings using a sparse set of observed ratings. We can represent these data as a matrix of $N$ users and $M$ movies, with a subset of known elements, measured on a fixed scale, e.g.,~\mbox{1--5}.

One popular approach to approximate the observed matrix is to estiate its low-rank decomposition \citep{srebro2005generalization}. In the low-rank representation $WA'$, columns~of~$A$ spanning the space of movies can be \tr{loosely associated with some implicit (latent) characteristics such as taste, style, or genre}, and each rater is represented as a weighted sum of their preferred characteristics, i.e., a row in the matrix of latent variables $W$ (see Figure~\ref{fig:idea}).

We can use the same idea to predict sparsely sampled curves, as long as we introduce an additional smoothing step. The low-dimensional latent structure now corresponds to progression patterns, and a trajectory of each individual can be represented as a weighted sum of these ``principal'' patterns. In Figure \ref{fig:idea}, the patterns are given by $A'B'$, while the individual weights are encoded in $W$.

\begin{figure}[h]
  \centering
  \includegraphics[width=0.9\linewidth]{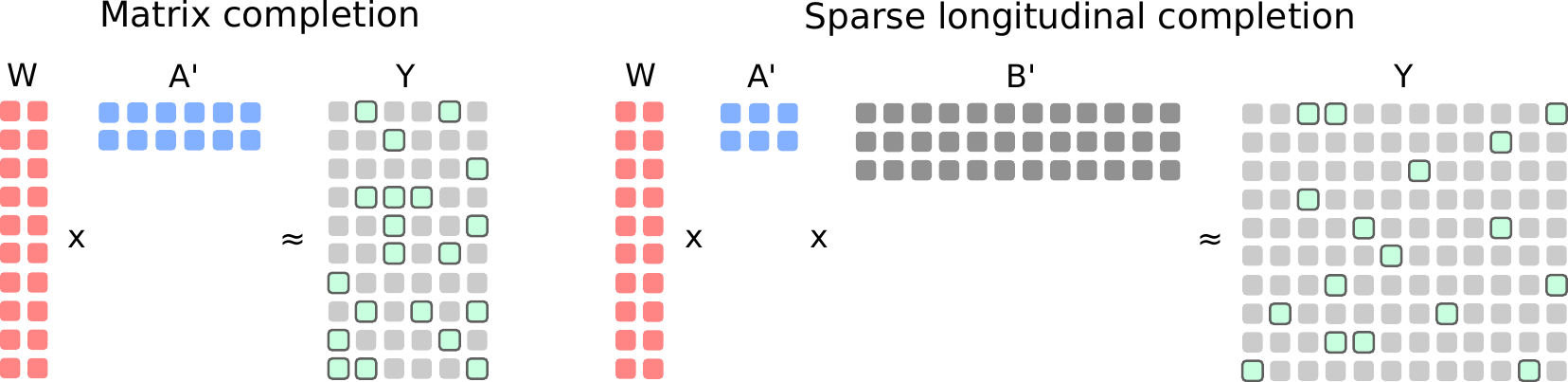}
  \caption{The key observation motivating this work is the fact that the problem of estimating trajectories can be mapped to a matrix completion problem where we estimate values of unobserved entries of a matrix. Matrix completion can be approached with matrix factorization where we look for $WA'$ of low rank, approximating observed values in $Y$ (circled green rectangles in the matrix $Y$). In the sparse longitudinal setting, we enforce smoothness by fixing the basis $B$ (e.g., splines, here with three elements), and again we find a low-rank matrix $W'A$, such that $W'AB$ approximates observed values in~$Y$.}
  \label{fig:idea}
\end{figure}

We first introduce the methodology for univariate sparsely-sampled processes. The direct method, mixed-effect models, and low-rank approximations described in Section \ref{s:background} have their counterparts in the matrix completion setting. We discuss these analogies in sections \ref{ss:direct-matrix} and \ref{ss:low-rank-matrix}. Next, in Section \ref{s:multivariate}, we show that the simple representation of the problem allows for extension to multivariate sparsely-sampled processes and a regression setting.

\subsection{Notation}\label{ss:notation}


For each individual $i \in \{1,2,...,N\}$ we observe $n_i$ measurements $\{\tilde y_{i,1},...,\tilde y_{i,n_i}\}$ at time-points $\{t_{i,1},t_{i,2},...,t_{i,n_i}\}$. Unlike in the prior work introduced in Section \ref{s:background} operating on entire curves, here we discretize time to $T$ time-points $G = \left[\tau_1, \tau_2, ..., \tau_T\right],$ where $t_{\min} = \tau_1$, $t_{\max} = \tau_T$ and $T$ is arbitrarily large. Each individual $i$ is expressed as a partially observed vector $r_i \in \R^T$. For each time-point $t_{i,j}$ for $1 \leq j \leq n_i$ we find a corresponding grid-point $g_i(j) = \argmin_{1 \leq k \leq T}  |\tau_k - t_{i,j}|$. We assign $y_{i,g_i(j)} = \tilde y_{i,j}$. Let $O_i = \{g_i(j): 1 \leq j \leq n_i \}$ be a set of grid indices corresponding to observed grid-points for an individual $i$. All elements outside $O_i$, i.e. $\{y_{i,j} : j \notin O_i\}$ are considered missing.


For $T$ sufficiently large, our observations can be uniquely represented as a $N \times T$ matrix $Y$ with missing values. We denote the set of all observed elements by pairs of indices as $\Omega = \{ (i,j) : i\in \{1,2,...,N\}, j \in O_i \}$. Let $P_\Omega(Y)$ be the projection onto observed indices, i.e. $P_\Omega(Y) = W$, such that $W_{i,j} = Y_{i,j}$ for $(i,j) \in \Omega$ and $W_{i,j} = 0$ otherwise. We define $P^\perp_\Omega(Y)$ as the projection on the complement of $\Omega$, i.e. $P^\perp_\Omega(Y) = Y - P_\Omega(Y)$. We use $\|\cdot\|_F$ to denote the Frobenius norm, i.e. the square root of the sum of matrix elements, and $\|\cdot\|_*$ to denote the nuclear norm, i.e. the sum of singular values.

As in Section \ref{s:background} we imposed smoothness by using a continuous basis $\bb(t) = [b_1(t),b_2(t),...,b_K(t)]'$. When we evaluate the basis on a grid $G$ we get a $T \times K$ matrix $B = [\bb(\tau_1),\bb(\tau_2),...,\bb(\tau_T)]'$. 

In our algorithms we use a diagonal thresholding operators defined as follows. Let $D~=~\diag(d_1,...,d_p)$ be a diagonal matrix. We define {\em soft-thresholding} as
\begin{equation*}
D_\lambda = \diag((d_1 - \lambda)_+,(d_2 - \lambda)_+,...,(d_p - \lambda)_+),\label{eq:thresholding}
\end{equation*}
where $(x)_+ = \max(x, 0)$, and {\em hard-thresholding} as
\begin{equation*}
D_\lambda^H = \diag(d_1,d_2,...,d_q,0,...,0),\label{eq:hard-thresholding}
\end{equation*}
where $q = \argmin_k(d_k < \lambda)$.

\subsection{Direct approach}\label{ss:direct-matrix}

The optimization problem \eqref{eq:direct-individual} of the direct approach described in Section \ref{ss:direct} can be approximated in the matrix completion setting. First, the bias introduced by the grid is negligible if the grid is sufficiently large and because measured processes are continouous. \tr{We approximate each observation $y_{i,j}$ with a point on the selected grid $y_{i,g(j)} \sim y_{i,j}$}. Next, we rewrite the optimization problem \eqref{eq:direct-individual} as a matrix completion problem
\begin{align}
 \argmin_{\{\bw_i\}}\sum_{i=1}^N \sum_{j=1}^{n_i}\left|y_{i,g_i(j)} - \bw_i' \bb(\tau_{g_i(j)}))\right|^2 &= \argmin_{\{\bw_i\}}\sum_{(i,k) \in \Omega}\left|y_{i,k} - \bw_i' \bb(\tau_{k}))\right|^2\nonumber\\
&= \argmin_W \| P_\Omega(Y - WB') \|_F^2,\label{eq:direct-matrix}
\end{align}
where by optimization over $\{\bw_i\}$ we mean optimization over all $\{\bw_i : i \in \{ 1,2,...,N \}\}$ and $W$ is an $N \times K$ matrix composed of vectors $\{\bw_i'\}$ stacked vertically.

The matrix formulation in equation \eqref{eq:direct-matrix} and the classic approach in Section \ref{ss:direct} share multiple characteristics. In both cases, if data is dense, we may find an accurate representation of the underlying process simply by fitting least-squares estimates of $W$ or $\{\bw_i\}$. Conversely, if the data is too sparse, the problem becomes ill-posed, and the least-squares estimates can overfit. 

However, representations \eqref{eq:direct-individual} and \eqref{eq:direct-matrix} differ algebraically and this difference constitutes the foundation for the method introduced in this paper. The matrix representation enables us to use the matrix completion framework and, in particular, in Section \ref{ss:matrix-factorization} we introduce convex optimization algorithms for solving \eqref{eq:direct-matrix}. 

Some low-rank constraints on the random effect from the mixed-effect model introduced in Section \ref{ss:lmm} can be expressed in terms of constraints on $W$. In particular in Section \ref{ss:low-rank-matrix} we show that the linear mixed-effect model can be expressed by constraining the $\rank(W)$. 

\subsection{Low-rank matrix approximation}\label{ss:low-rank-matrix}
In the low-rank approach described in Section \ref{ss:reduced-rank} we assume that individual trajectories can be represented in a low-dimensional space, by constraining the rank of $W$.

We use a similar approach in solving \eqref{eq:direct-matrix}. One difficulty comes from the fact that optimization with a rank constraint on $W$ turns the original least squares problem into a non-convex problem. Motivated by the matrix completion literature, we relax the rank constraint in \eqref{eq:direct-matrix} to a nuclear norm constraint and we attempt to solve
\begin{align}
  \argmin_W \| P_\Omega(Y - WB') \|_F^2 + \lambda\|W\|_*,
\label{eq:rank-restricted}
\end{align}
for some parameter $\lambda > 0$.

From the practical standpoint, low-rank representation corresponds to describing curves in a basis comprised of only a few functions which are the best functions for obtaining low reconstruction errors. In Figure \ref{fig:data-components} we illustrate two functions for describing processes introduced in Section \ref{s:motivation} and Figure \ref{fig:motivation}. Such low-dimensional representation of processes can also be used for clustering based on the type of their progression as we discuss in the data study (Section~\ref{s:data-study}).

\begin{figure}[ht!]
  \includegraphics[width=0.49\linewidth]{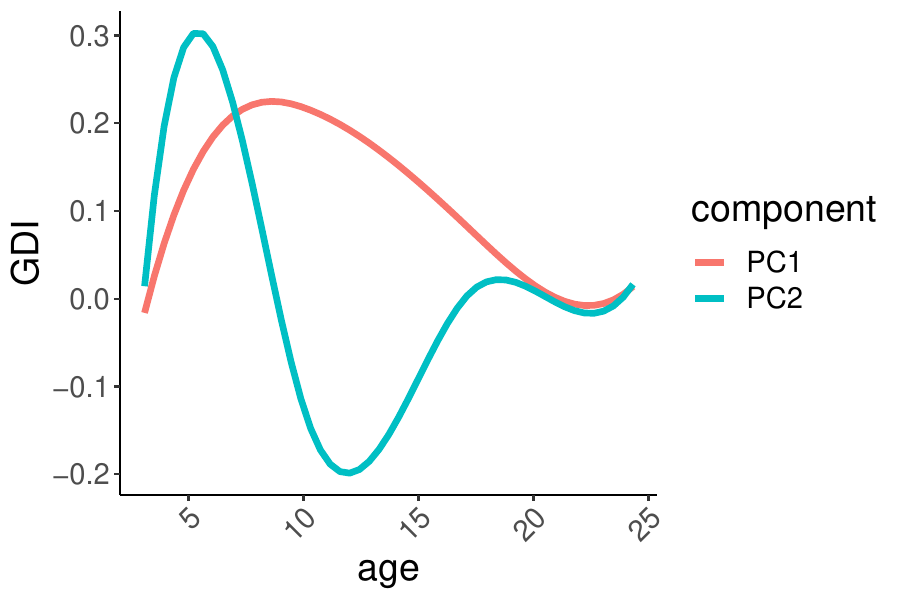}
  \includegraphics[width=0.49\linewidth]{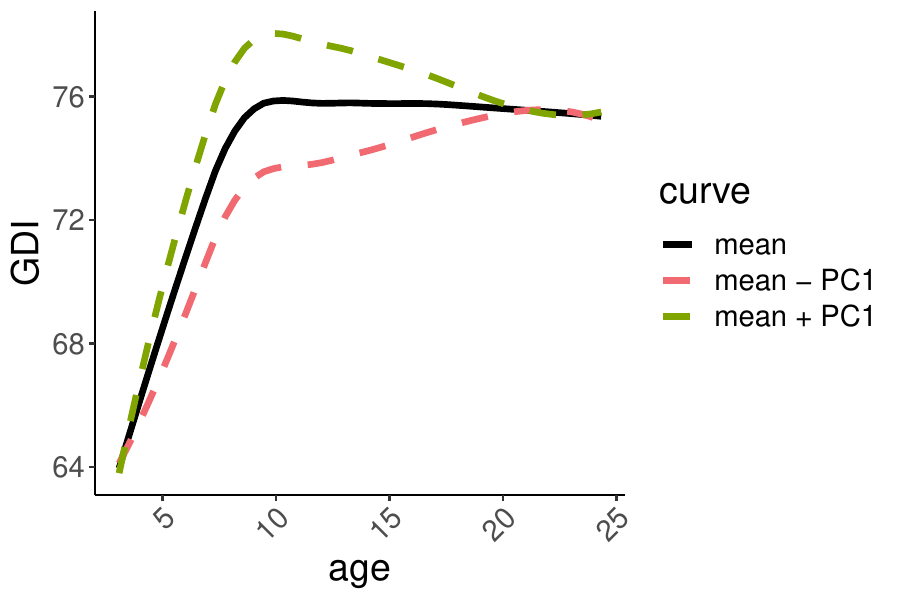}
\caption{Left: The first two trends of variability for processes of the Gait Deviation Index, derived from the first two dimensions of the low-rank decomposition of the observed matrix. We can appreciate variability around the age of 10 in the first component, and correction between and after age of 10 in the second component. Right: We illustrate how the individual score on the first component (positive and negative value) affects the population mean.}
    \label{fig:data-components}
\end{figure}

\eqref{eq:rank-restricted}, \citet{cai2010singular} propose a first-order singular value thresholding algorithm (SVT), for solving a general class of problems involving a nuclear norm penalty and a linear form of $Y$, which includes \eqref{eq:rank-restricted}. Their algorithm can handle large datasets and has strong guarantees on convergence, but it requires tuning the step size parameter, which can greatly influence performance. This limitation was addressed by \citet{ma2011fixed,mazumder2010spectral,hastie2015matrix} who introduced iterative procedures which do not depend on such tuning. Moreover, \citet{hardt2014fast,chen2015fast} propose methods for the non-convex problem. \citet{ge2016matrix} argue that the positive semidefinite matrix completion problem\tr{, where the target matrix is approximated with a positive semidefinite matrix,} has no spurious local minima\tr{, i.e. all local minima must also be global}.

In the last decade, machine learning and statistical learning communities have introduced multiple algorithms for matrix completion. Many of them are suitable for solving \eqref{eq:rank-restricted}. However, in this article we focus on analyzing the benefits of framing the trajectory prediction problem \eqref{eq:direct-individual} in the matrix completion framework, rather than on benchmarking possible solutions.

\subsection{Low-rank approximation with singular value thresholding}\label{ss:matrix-factorization}

The low-rank probabilistic approach, introduced in Section \ref{ss:reduced-rank}, is based on the observation that the underlying processes for each subject can be approximated in a low-dimensional space. 
Here, we exploit the same characteristic using matrix-factorization techniques for solving the optimization problem \eqref{eq:rank-restricted}.


For the sake of clarity and simplicity, we choose to solve the problem \eqref{eq:rank-restricted} with an extended version of the \textsc{Soft-Impute} algorithm designed by \citet{hastie2015matrix,mazumder2010spectral}. As discussed in Section \ref{ss:low-rank-matrix}, many other convex optimization algorithms can be applied.

The key component to the solution is the following property linking problem \eqref{eq:rank-restricted} with the Singular Value Decomposition (SVD). Consider the fully observed case of \eqref{eq:rank-restricted}. Using Theorem 2.1 in \citet{cai2010singular}, one can show that the optimization problem
\begin{align}\label{eq:optsvd}
\argmin_{W} \frac{1}{2} \| Y - WB' \|_F^2 + \lambda\|W\|_*,
\end{align}
where $B'B= I$, has a unique solution $W = \cS_\lambda (YB)$, where $\cS_\lambda(X) = UD_\lambda V'$ and $X = UDV'$ is the SVD of $X$. We refer to $\cS_\lambda(X)$ as the singular value thresholding (SVT) of $X$.

In order to solve \eqref{eq:optsvd} with a sparsely observed  $Y$, we modify the \textsc{Soft-Impute} algorithm  to account for the basis $B$. Our Algorithm \ref{alg:soft-impute} iteratively constructs better approximations of the global solution for each $\lambda$ in some predefined set $\{\lambda_1, \lambda_2, ..., \lambda_k\}$ for $k \in \N$. For a given approximation of the solution $W^{old}$, we use $W^{old}B'$ to impute unknown elements of $Y$ obtaining $\tilde{Y}$. Then, we construct the next approximation $W^{new}$ by computing SVT of $\tilde{Y}$.


As a stopping criterion, we compute the change between subsequent solution, relative to the magnitude of the solution, following the methodology in \cite{cai2010singular}. We set a fixed threshold of $\varepsilon > 0$ for this criterion, depending on the desired accuracy.

\begin{algorithm}
\caption{\textsc{Soft-Longitudinal-Impute}\label{alg:soft-impute}}
\begin{enumerate}
\item Initialize $W^{old}$ with zeros
\item Do for $\lambda_1 > \lambda_2 > ... > \lambda_k$:
\begin{enumerate}
\item Repeat:
\begin{enumerate}
\item Compute $W^{new} \leftarrow S_{\lambda_i}( (P_\Omega(Y) + P_\Omega^\perp(W^{old}B'))B )$
\item If $\frac{\|W^{new} - W^{old}\|_F^2}{\|W^{old}\|_F^2} < \varepsilon$ exit
\item Assign $W^{old} \leftarrow W^{new}$
\end{enumerate}
\item Assign $\hat{W}_{\lambda_i} \leftarrow W^{new}$
\end{enumerate}
\item Output $\hat{W}_{\lambda_1}, \hat{W}_{\lambda_2}, ... , \hat{W}_{\lambda_k}$
\end{enumerate}
\end{algorithm}

A common bottleneck of the algorithms introduced by \citet{cai2010singular,mazumder2010spectral,ma2011fixed} as well as other SVT-based approaches is the computation of the SVD of large matrices. This problem is particularly severe in standard matrix completion settings, such as the Netflix problem, where the matrix size is over $400{,}000 \times 20{,}000$. However, in our problem,
\begin{align}\label{eq:small-rank}
  \rank(WB') \leq \rank(B) = K \ll N,
\end{align}
and $K \sim 10$ in our motivating data example. While algorithms for large matrices are applicable here, the property \eqref{eq:small-rank} makes the computation of SVD feasible in practice with generic packages such as PROPACK \citep{larsen2004propack}.


For estimating curves on new data, we need to find $\bw_i = V\ba_i$ minimizing the prediction error. To this end, we regress observed values $Y_i$ on $B_iV$ with a penalty $\frac{\lambda}{2}\|\ba_i\|^2$ on the coeffcients, i.e. we optimize
\begin{align*}
\min_{\ba_i} \|Y_i - B_i V \ba_i\|^2 + \frac{\lambda}{2}\|\ba_i\|^2,
\end{align*}
where $B_i$ are basis functions evaluated at the corresponding timepoints.


\tr{While our algorithms for finding $W$ do not depend on the density of the grid, it is important to highlight that with the growing $T$ the norm of the solution $W_T$ grows to infinity and $W_T/\sqrt{T} \rightarrow W_0$, for some fixed $W_0$ (See Appendix \ref{s:convergence}). In order to obtain comparable results for different $T$, parameters $\lambda$ need to be adjusted to this scale (for example, by modifying the penalty to $\lambda\sqrt{T}\|W\|_*$).}

The algorithm converges at a rate $\frac{1}{t},$ where $t$ is the number of iterations (this can be established by an elementary extension of the proofs in \citet{mazumder2010spectral}). We demonstrate a formal mapping of our problem to their setting in Appendix \ref{s:convergence}.

\subsection{$l_0$ regularization}

While the nuclear norm relaxation \eqref{eq:rank-restricted} is motivated by making the problem convex, \citet{mazumder2010spectral} argue that in many cases it can also give better results than the rank constraint. They draw an analogy to the relation between best-subset regression ($\ell_0$ regularization) and LASSO ($\ell_1$ regularization as in \citet{tibshirani1996regression, friedman2001elements}). In LASSO, by shrinking model parameters, we allow more parameters to be included, \tr{which} can potentially improve predictions if the true subset is larger than the one derived through $\ell_0$ regularization. In the case of \eqref{eq:rank-restricted}, shrinking the nuclear norm allows us to include more dimensions of $W$ again potentially improving the prediction if the true dimension is high.


Conversely, the same phenomenon can also lead to problems if the underlying dimension is small. In such case\tr{s}, shrinking may cause inclusion of unnecessary dimensions emerging from noise. To remedy this issue, following the analogy with penalized linear regression, we may consider another class of penalties. In particular, we may consider coming back to the rank constraint by modifying the nuclear norm penalty \eqref{eq:rank-restricted} to $\ell_0$. We define $\|W\|_0 = \rank(W)$. The problem
\begin{align}\label{eq:matrixproblem-final-l0}
\min_{W} \frac{1}{2} \|P_\Omega(Y - WB')\|_F^2 + \lambda\|W\|_0,
\end{align}
has a solution $W = \cS_\lambda^H (YB)$, where $\cS_\lambda^H(X) = UD_\lambda^H V'$ and $X = UDV'$ is the SVD of~$X$. We refer to $\cS_\lambda^H(X)$ as the hard singular value thresholding (hard SVT) of $X$. We use Algorithm \ref{alg:hard-impute} to find the hard SVT of $YB$.


\begin{algorithm}
\caption{\textsc{Hard-Longitudinal-Impute}\label{alg:hard-impute}}
\begin{enumerate}
\item Initialize $W^{old}_{\lambda_i}$ with solutions $\tilde{W}_{\lambda_i}$ from \textsc{Soft-Longitudinal-Impute}
\item Do for $\lambda_1 > \lambda_2 > ... > \lambda_k$:
\begin{enumerate}
\item Repeat:
\begin{enumerate}
\item Compute $W^{new} \leftarrow S_{\lambda_i}^H( (P_\Omega(Y) + P_\Omega^\perp(W^{old}B'))B )$
\item If $\frac{\|W^{new} - W^{old}\|_F^2}{\|W^{old}\|_F^2} < \varepsilon$ exit
\item Assign $W^{old} \leftarrow W^{new}$
\end{enumerate}
\item Assign $\hat{W}_{\lambda_i} \leftarrow W^{new}$
\end{enumerate}
\item Output $\hat{W}_{\lambda_1}, \hat{W}_{\lambda_2}, ... , \hat{W}_{\lambda_k}$
\end{enumerate}
\end{algorithm}

The problem \eqref{eq:matrixproblem-final-l0} is non-convex, however, by starting from the solutions of Algorithm~\ref{alg:soft-impute} we explore the space around the global minimum of the $\ell_1$ version of the problem. \citet{ge2016matrix} show that this strategy is successful empirically.


\subsection{Dense grid and proximal gradient descent}\label{ss:continous-support}

Growing grid size parameter $T$ improves accuracy of approximation of the data on the grid, but benefits in approximation are negligible when $T$ becomes sufficiently large. However, we show that $T \rightarrow \infty$ also leads to a continous version of the algorithm, equivalent to proximal gradient descent method applied to the initial formulation of the problem \eqref{eq:direct-individual} with nuclear norm penalty.

We start by observing that in Algorithm \ref{alg:soft-impute} the Step 2(a)(i) depends on the grid and requires multiplication with all evaluations of the basis on that grid. We can rewrite the expression in $S_{\lambda_i}$ such that only a sparse number of evaluations of the basis is required. Indeed, we write
\begin{align}
(P_\Omega(Y) + P_\Omega^\perp(W^{old}B'))B &= (P_\Omega(Y) - P_\Omega(W^{old}B') + W^{old}B')B\nonumber\\
&= (P_\Omega(Y - W^{old}B'))B + W^{old}\label{eq:rearrangement}
\end{align}
and to compute $W^{old}B'$ on $\Omega$ in \eqref{eq:rearrangement} we only need $W^{old}$ and an evaluation of the basis on gridpoints corresponding to $\Omega$. Therefore, for computing $(P_\Omega(Y - W^{old}B'))B$ we also only need $B$ evaluated on gridpoints with a nonzero number of observations. Now, if we increase the number of grid points to infinity we derive a solution where the basis needs to be evaluated only at observed timepoints.

Since $B$ is orthogonal, when $T \rightarrow \infty$ the product $(P_\Omega(Y - W^{old}B'))B$ converges to~$0$ and the update in Step 2(a)(i) in Algorithm \ref{alg:soft-impute} also converges to~$0$ at the rate $1/T$. Therefore, in order to obtain comparable results for different $T$ we need to scale the update direction $(P_\Omega(Y - W^{old}B'))B$ with $\alpha / T$, for some $\alpha > 0$. We refer to $\alpha$ as the step size.

We show that at the limit $T\rightarrow\infty$ our approach is equivalent to proximal gradient descent. Let us define the loss function as in \eqref{eq:direct-individual}, i.e.
\begin{align*}
 f(W) = \sum_{i=1}^N \sum_{j=1}^{n_i}\left|y_{i,j} - \bw_i' \bb(\tau_{i,j}))\right|^2, 
\end{align*}
where $\tau_{i,j}$ are timpoints, $y_{i,j}$ are observations. We optimize the loss function $f$ with a nuclear norm penalty for some $\lambda > 0$
\begin{align}
\argmin_{W}f(W) + \lambda\|W\|_*,\label{eq:objective}
\end{align}
where $W = [\bw_1', \bw_2', ..., \bw_N']'$. To solve problem \eqref{eq:objective} we define
\begin{align*}
 G_i = \frac{\partial}{\partial \bw_i} f(W) =-2\sum_{j=1}^{n_i} (y_{i,j} - \bw_i' \bb(\tau_{i,j}))\bb'(\tau_{i,j})
\end{align*}
and the gradient of $f$ takes form
\begin{align*}
 \nabla_W f(W) &= [G_1', G_2', ..., G_N']'.
\end{align*}

Note that $\nabla_W f(W)$ is corresponding to $(P_\Omega(Y - W^{old}B'))B$ in \eqref{eq:rearrangement} with infinitely dense grid. For the update step in the proximal gradient descent method, we use step size $\alpha > 0$ and singular value thresholding with a threshold $\alpha\lambda$ as recommended in \citet{hastie2015statistical}. We describe the full procedure in Algorithm \ref{alg:grid-free}.

\begin{algorithm}
\caption{\textsc{Grid-free soft-Longitudinal-Impute}\label{alg:grid-free}}
\begin{enumerate}
\item Initialize $W^{old}$ with zeros
\item Do for $\lambda_1 > \lambda_2 > ... > \lambda_k$:
\begin{enumerate}
\item Repeat:
\begin{enumerate}
\item Compute $W^{new} := S_{\alpha\lambda}(W^{old} - \alpha \nabla_W f(W^{old}))$
\item If $\frac{\|W^{new} - W^{old}\|_F^2}{\|W^{old}\|_F^2} < \varepsilon$ exit
\item Assign $W^{old} \leftarrow W^{new}$
\end{enumerate}
\item Assign $\hat{W}_{\lambda_i} \leftarrow W^{new}$
\end{enumerate}
\item Output $\hat{W}_{\lambda_1}, \hat{W}_{\lambda_2}, ... , \hat{W}_{\lambda_k}$
\end{enumerate}
\end{algorithm}

\citet{nesterov2013gradient} showed sufficient conditions for convergence of Algorithm \ref{alg:grid-free}. If $\alpha \in (0, 1/L]$ where $L$ is the Lipschitz constant of $\nabla_W f$, then the algorithm converges at the rate $1/t$, where $t$ is the number of steps.

\subsection{A link between the reduced-rank model and Soft-Longitudinal-Impute}\label{s:the-link}


Intuitively, we might expect similarity between the principal directions derived using the probabilistic approach \eqref{eq:probabilistic} and their counterparts derived from the SVT-based approach. We investigate this relation by analyzing behavior of SVT for matrices sampled from the probabilistic model given by \eqref{eq:probabilistic}.

For simplicity, let us assume that $\mu = 0$ and the data is fully observed on a grid $G$ of $T$ time-points. Assume that observations $i \in \{1,2,...,N\}$ come from the mixed-effect model
\begin{align}
  \by_i | \bw_i &\sim \cN( B \bw_i, \sigma^2 I_T), \label{eq:model-y}\\
  \bw_i &\sim \cN(0 , \Sigma), \nonumber
\end{align}
%
where $\Sigma$ is an unknown covariance matrix of rank $q < K$ and variables $\{\bw_i, \by_i\}$ are independent. By the spectral decomposition theorem we decompose $\Sigma = V\Lambda V'$, where $V$ is a $K\times K$ orthogonal and $\Lambda$ is a diagonal $K\times K$ matrix with $q$ positive coefficients in decreasing order.
Since $\by_i$ and $\bw_i$ are independent, the distribution \eqref{eq:model-y} can be rewritten as
\begin{align}
  \by_i &\sim \cN(0, B V \Lambda V'  B' + \sigma^2 I_T). \label{eq:model}
\end{align}
The model \eqref{eq:model} is a factor model with uncorrelated errors and $q$ factors---the first $q$ columns of $BV$.

The following theorem constitutes a link between the mixed-effect model and SVT \tr{(Theorem 9.4.1 in \citet{mardia1980multivariate})},

\begin{theorem}\label{thm:maxlike}
  Let $Y = [\by_1,...,\by_N]'$ be the observed matrix and let $S_{\sigma^2}(YB) = UD_{\sigma^2}Q'$. Then, $(D_{\sigma^2},Q)$ is the maximum likelihood estimator of $(\Lambda,V)$.
\end{theorem}

\tr{Theorem \ref{thm:maxlike} implies that in the fully observed case the estimator of $(\Lambda,V)$ converges to the true parameter at the rate $1 / \sqrt{n}$, as the MLE.

Factor analysis methods give not only estimates of $\Lambda$ and $V$ but also \tr{predictors} of the individual latent variables $W = [\bw_1,...,\bw_N]'$.
In the multivariate analysis literature, there are multiple ways to estimate factor scores, i.e., a matrix $A$ such that $X \sim AD_{\sigma^2}V'$. Most notably researcher introduce Spearman's scores, Bartlett's scores, and Thompson's scores \citep{kim1978factor}. 
Simply taking $W = U$ as the estimate of the scores corresponds to the solution of \eqref{eq:optsvd} as long as $\lambda = \sigma^2$.}

In Theorem \ref{thm:maxlike} we assume that $\sigma^2$ is known, which is rarely the case in practice. However, the likelihood of $(V,\sigma)$ can be parametrized by $\sigma$, and we can find the optimal solution analytically. This corresponds to minimizing \eqref{eq:optsvd} for different $\lambda$.


This relation is also confirmed in our simulation study in Section \ref{s:simulation} (see Figure \ref{fig:principal-components}). A similar analogy is drawn between the classical principal component analysis and probabilistic principal component analysis by \citet{tipping1999probabilistic} and \citet{james2000principal}.

\section{Multivariate longitudinal data}\label{s:multivariate}

In practice, we are often interested in the prediction of a univariate process in the context of other longitudinal measurements and covariates constant over time. Examples include prediction of disease progression given patient's demographics, data from clinical visits at which multiple blood tests or physical exams are taken, or measurements which are intrinsically multivariate such as gene expression or x-rays collected over time. The growing interest in this setting stimulated research in latent models \citep{sammel1996latent} and multivariate longitudinal regression \citep{gray1998estimating,gray2000multidimensional}. \citet{diggle2002analysis} present an example case study in which longitudinal multivariate modeling enables estimation of joint confidence region of multiple parameters changing over time, shrinking the individual confidence intervals.

In this section, we present an extension of our univariate framework to multivariate measurements sparse in time. We explore two cases: (1) dimensionality reduction, where we project sparsely sampled multivariate processes to a small latent space, and (2) linear regression, where we use a multivariate process and covariates constant over time for prediction of a univariate process of interest. To motivate the methodology, we start with a regression involving two variables observed over time.


\subsection{Illustrative example: Univariate regression}



Suppose that the true processes are in a low-dimensional space of some continuous functions (e.g., splines) and that we observe them with noise. More precisely, let
\begin{align}\label{eq:definitions-xy}
  \bx_i = B\bw_i + \be_{x,i} \text{\ \ and\ \ } \by_i = B\bu_i + \be_{y,i},
\end{align}
for $1 \leq i \leq N$, where $\bx_i,\by_i,\be_{x,i},\be_{y,i}$ are $T \times 1$ vectors, $\bw_i, \bu_i$ are $K \times 1$ vectors and $B$ is a $T \times K$ matrix of $K$ splines evaluated on a grid of $T$ points. We assume zero-mean independent errors $\be_{x,i},\be_{y,i}$ with fixed covariance matrices $\Sigma_X,\Sigma_Y$ respectively, and that the true processes underlying the observed $\bx_i$ and $\by_i$ follow a linear relation in the spline space, i.e.
\begin{align}\label{eq:linear-xy}
  \bu_i = A'\bw_i,
\end{align}
where $A$ is an unknown $K \times K$ matrix.

Let $X,Y,U,W$ be matrices formed from $\bx_i',\by_i',\bw_i',\bu_i'$ stacked vertically. From \eqref{eq:linear-xy} we have $U = WA$, while \eqref{eq:definitions-xy} implies
\begin{align}\label{eq:almost-errors-in-variables}
X = WB' + E_X \text{ \ and \ } Y = WA B' + E_{Y},
\end{align}
where $E_X,E_Y$ are matrices of observation noise. Without loss of generality we assume that $B$ is orthonormal. We have full freedom to choose the basis $B$, since any basis can be orthogonalized using, for example, the singular value decomposition. 

Due to orthogonality of $B$ and after multiplying both expressions in \eqref{eq:almost-errors-in-variables} by $B$ we can map the problem to the classical multivariate {\em errors-in-variables models}. 
Let
\begin{align}\label{eq:errors-in-variables}
  \tilde{X} = XB = W + \tilde{E}_X \text{ and } \tilde{Y} = YB = WA + \tilde{E}_Y,
\end{align}  
where $\tilde{E}_X = E_XB$ and $\tilde{E}_Y = E_YB$. In errors-in-variables models it is assumed that the parameters $W$ and $A$ are unknown, and are to be estimated. Both regressors and responses are measured with noise (here $\tilde{E}_X$ and $\tilde{E}_Y$). The parameter $W$ can be interpreted as a latent representation of both $\tilde{X}$ and $\tilde{Y}$.

The problem of estimating parameters in \eqref{eq:errors-in-variables} has been extensively studied in literature dating back to \citet{adcock1878problem}. Two main methods for estimating parameters in \eqref{eq:errors-in-variables} are {\em maximum likelihood estimates (MLE)} and {\em generalized least squares estimators (GLSE)}. The estimators in MLE are derived under the assumption of certain distributions of the errors. In GLSE, the only assumption about errors is that they are independent, zero-mean, and they have a common covariance matrix. Then, $\tilde{X} - W$ and $\tilde{Y} - WA$ are zero-mean and estimates for $W$ and $B$ can be found by optimizing some norm of these expressions. \citet{gleser1973estimation} show that in the no-intercept model for $\tilde{X}$ and $\tilde{Y}$ of the same size (as in our case) and under the assumption of Gaussian errors, MLE and GLSE give the same estimates of $A$, if GLSE are derived for the Frobenius norm. 

In this work, we focus on the GLSE approach as it can be directly solved in our matrix factorization framework and we find it easier to deploy and extend in practice. 
To account for different magnitudes of the noise in $X$ and $Y$, we consider the optimization problem with weights
\begin{align}\label{eq:lag:bivariate}
  \minimize_{A,W} & \frac{1}{\sigma_X^2}\|XB - W \|_F^2 + \frac{1}{\sigma_Y^2} \| YB - WA \|_F^2,
\end{align}
where $\sigma_X,\sigma_Y > 0$. Let $\gamma = \sigma_X^2 / \sigma_Y^2$. Then \eqref{eq:lag:bivariate} can be transformed to
\begin{align}\label{eq:lag:bivariate3}
  \minimize_{A,W} & \| (XB : \gamma YB) - W(I: \gamma A) \|_F^2,
\end{align}
where $(\cdot : \cdot)$ operator stacks \tr{horizontally} matrices with the same number of rows. To solve \eqref{eq:lag:bivariate3}, we show that the SVD of $(XB : \gamma YB)$ truncated to the first $K$ dimensions, can be decomposed to $W(I: \gamma A)$. Let $USV'$ be the SVD of $(XB : \gamma YB)$, with
\[
U = \begin{bmatrix}
  U_1 & U_2
\end{bmatrix},
S = \begin{bmatrix}
      S_{11} & S_{12}\\
      S_{21} & S_{22}
\end{bmatrix} 
\text{ and }
V = \begin{bmatrix}
      V_{11} & V_{12}\\
      V_{21} & V_{22}
    \end{bmatrix},
\]
where each $S_{ij}$ and $V_{ij}$ is a $K\times K$ matrix for $1 \leq i,j \leq 2$ and each $U_i$ is a $N \times K$ matrix for $1 \leq i \leq 2$. By Lemma 2.1 and Lemma 2.2 in \citep{gleser1981estimation} matrix $V_{11}$ is almost surely nonsingular. Therefore, $V_{11}^{-1}$ almost surely exists and we can transform the decomposition such that $(I : \gamma A) = (V_{11}')^{-1}\begin{bmatrix}V_{11}' & V_{21}'\end{bmatrix}$ and $W = U_1 S_{11} V_{11}'$, solving \eqref{eq:lag:bivariate3}.

For partially observed data, if they are very sparse, it might be essential to constrain the rank of the solution. The partially-observed and rank-constrained version of the problem~\eqref{eq:lag:bivariate3} takes the form 
\begin{align*}
  \minimize_{A,W} & \| P_{\tilde\Omega}((X:\gamma Y) - W(B':\gamma A B')) \|_F^2,\nonumber\\
  \subjectto & \rank(W(B':\gamma A B')) = k,
\end{align*}
where $k$ is the desired rank of the solution and $P_{\tilde\Omega}$ is a projection on \[\tilde\Omega = \{(q,r): (q,r) \in \Omega \text{ or } (q,r-T) \in \Omega\}. \] As previously, for an unknown $k$ we can introduce a rank penalty using the nuclear norm
\begin{align}\label{eq:lag:bivariate2-partial}
  \minimize_{A,W} & \| P_{\tilde\Omega}((X:\gamma Y) - W(B':\gamma A B')) \|_F^2 + \lambda\|W(B':\gamma A B')\|_*.
\end{align}
The algorithm in the general case of multiple processes is derived in Section \ref{ss:dim-red}.

Although we motivate the problem as a regression of $Y$ on $X$, $X$ and $Y$ are symmetric in \eqref{eq:lag:bivariate3}. The low-rank matrix $W$ is, therefore, a joint low-rank representation of matrices $X$ and $Y$ and thus our method can be seen as a dimensionality reduction technique or as a latent space model. In Section \ref{ss:dim-red} we extended this idea to a larger number of variables. In Section~\ref{ss:regression} we discuss how this approach can be used for regression.

The linear structure of \eqref{eq:almost-errors-in-variables} allows us to draw analogy not only to the errors-in-variables models but also to the vast literature on {\em canonical correlation analysis (CCA)}, {\em partial least squares (PLS)}, {\em factor analysis (FA)}, and {\em linear functional equation (LFE) models}. \cite{borga1997unified} show that solutions of CCA and PLS can also be derived from the SVD of stacked matrices, as we did with $(XB:\gamma YB)$ in \eqref{eq:lag:bivariate3}. \cite{gleser1981estimation} thoroughly discusses the relation between errors-in-variables, FA, and LFE.


Finally, the method of using SVD for stacked matrices has also been directly applied in the recommender systems context. \citet{condli1999bayesian} showed that for improving prediction of unknown entries in some partially observed matrix $Q$ one might consider a low-rank decomposition of $(Q:R)$, where $R$ are additional covariates for each row, e.g., demographic features of individuals. 


\subsection{Dimensionality reduction}\label{ss:dim-red}

Suppose that for every individual we observe multiple variables varying in time (e.g. results of multiple medical tests at different times in a clinic) and we want to find a projection on $\R^d$ maximizing the variance explained for some $d \in \N$. This projection would correspond to characterizing patients by their progression trends of multiple metrics simultaneously.


We extend the equation \eqref{eq:lag:bivariate3} to account for a larger number of covariates and we do not impose decomposition of the solution yet. We formulate the following optimization problem
\begin{align}\label{eq:dim-reduction}
  \argmin_{W} &\| (X_1B:X_2B:...:X_pB) - W \|_F^2 + \lambda\|W\|_*,
\end{align}
where $X_i$ are some $N \times T$ matrices corresponding to the processes measured on a grid of $T$ points, $B$ is a basis evaluated on the same grid and orthogonalized (a $T \times K$ matrix), and $W$ is a $N \times pK$ matrix.


Let $ \bB = I_p \otimes B $ be the Kronecker product of $p \times p$ identity matrix and $B$, i.e. a $pT \times pK$ matrix with $B$ stacked $p$ times on the diagonal, and let $\bX = (X_{1}: X_{2}:...: X_{p})$. Matrix $\bB$ is orthogonal and therefore results developed in Section \ref{ss:reduced-rank} apply here. In particular, singular value thresholding solves
\begin{align*}
\argmin_{W} &\| \bX - W \bB' \|_F^2 + \lambda\|W\|_*
\end{align*}
and we can use Algorithm \ref{alg:soft-impute} for solving
\begin{align}\label{eq:multivar-partially}
\argmin_{W} &\| P_{\Omega}\left(\bX - W\bB\right) \|_F^2 + \lambda\|W\|_*,
\end{align}
where $P_{\Omega}$ is the projection on observed indices $\Omega$.

The optimization problem \eqref{eq:dim-reduction} can be further extended. For example, we can chose a different basis for each process or scale individual contributions of processes using scaling factors. By modifying Equation \eqref{eq:multivar-partially} correspondingly, we can extend our solution to these cases.

\subsection{Regression}\label{ss:regression}

In practice, we are often interested in \tr{the relationship} between the progression of an individual parameter (e.g., cancer growth) and some individual features constant over time (e.g., demographics) or progressions of other covariates (e.g., blood tests, vitals, biopsy results). We show how our framework can be used to exploit relations between curves in order to improve estimation of progression of the main parameter of interest.

We start with a regression problem with fixed covariates and sparsely observed response trajectories. Assume that for each subject $1\leq i\leq N$ we observe covariates  $\bx_i \in \mathds{R}^d$ and samples from trajectories $\by_i\in \mathds{R}^{n_i}$ following the distribution
\[
\by_i | \bx_i \sim \mathcal{N}(B_iA\bx_i, \sigma^2 I_{n_i}),
\]
where $B_i$ is the matrix of evaluations of the basis on the same evaluation points as $\by_i$ (See Section \ref{ss:lmm}). We map the problem to the matrix completion framework as in Section \ref{s:context}. Let $X$ be a $N \times d$ matrix of observed covariates, $Y$ be a sparsely observed $N \times T$ matrix of trajectories, and $B$ be a $T \times K$ matrix representing a basis of $K$ splines evaluated on a grid of $T$ points. We consider the optimization problem
\begin{align}\label{regression}
 \argmin_A \| P_\Omega(Y - XAB')\|^2,
\end{align}
where $A$ is a $d \times K$ matrix and $P_\Omega$ is a projection on the observed indices $\Omega$. To solve \eqref{regression} we propose an iterative Algorithm \ref{alg:longitudinal-regression}.

\begin{algorithm}
\caption{\textsc{Sparse-Regression}\label{alg:longitudinal-regression}}
\vspace{3pt}
\begin{enumerate}
\item Initialize $A$ zeros
\item Repeat till convergence:
\begin{enumerate}
\item Impute regressed values $\hat{Y} = P_\Omega(Y) + P_\Omega^\perp(XAB')$
\item Compute $A^{new} \leftarrow (X'X)^{-1} X'\hat{Y}B$
\item If $\frac{\|A^{new} - A\|_F^2}{\|A\|_F^2} < \varepsilon$ exit
\item Assign $A \leftarrow A^{new}$
\end{enumerate}
\item Return $A$
\end{enumerate}
\end{algorithm}


Suppose we want to incorporate other variables varying in time for prediction of the response process. We can directly apply the method proposed in Section \ref{ss:dim-red} and model the response and regressors together. However, it might be suboptimal for prediction, as it optimizes the least-squares distance in all variables rather than only the response. This difference is analogous to the difference between regression line of some univariate $y$ on independent variables $x_1,...,x_p$ and the first principal component analysis of $(y,x_1,...,x_p)$, used for prediction of $y$. While the first one minimizes the distance to $y$, the latter minimizes the distance to $(y,x_1,...,x_p)$, which is usually less efficient for predicting $y$.

Alternatively, we can use methodology from Section \ref{ss:dim-red} only for covariates. The solution of \eqref{eq:multivar-partially} can be decomposed into $W = USV'$, and we regress $Y$ on $U$ as in \eqref{regression}.
Let $U$ be an $N \times d_2$ orthogonal matrix derived from $(X_1,...,X_p)$, where $d_2$ is the numbers of latent components. We search for a $d_2 \times K$ matrix $A$ solving
\begin{align}\label{eq:pcr}
\minimize_{A} \|P_\Omega(Y - UAB')\|_F^2 + \lambda\|A\|_*,
\end{align}
where $P_\Omega$ is a projection on a set of observed coefficients.
We propose a two-step iterative procedure (Algorithm~\ref{alg:sparse-regression}).



\begin{algorithm}
\caption{\textsc{Sparse-Longitudinal-Regression}\label{alg:sparse-regression}}
\vspace{3pt}
\begin{flushleft}
\textbf{Step 1: Latent representation}
\end{flushleft}
\begin{enumerate}
\item For sparsely observed $\bX = (X_1,X_2,...,X_p)$ find latent scores $U$ (Algorithm \ref{alg:soft-impute})
\end{enumerate}
\begin{flushleft}
\textbf{Step 2: Regression}
\end{flushleft}
\begin{enumerate}
\item For each $\lambda_1,\lambda_2,...,\lambda_k$
  \begin{itemize}
  \item Get $A_{\lambda_i}$ by solving the regression problem \eqref{eq:pcr} with $\textbf{Y},U,\lambda_i$ (Algorithm \ref{alg:longitudinal-regression})
  \end{itemize}
\item Return $A_{\lambda_1}, A_{\lambda_2}, ..., A_{\lambda_k}$
\end{enumerate}
\end{algorithm}


\section{Simulations}\label{s:simulation}

We illustrate the properties of the multivariate longitudinal fitting in a simulation study. First, we generate curves with quickly decaying eigenvalues of covariance matrices. Then, we compare the performance of the methods in terms of the variance explained by the predictions. 

\subsection{Data generation} Let $G$ be a grid of $T$ equidistributed points and let $B$ be a basis of $K$ spline functions evaluated on the grid $G$. 
We simulate three $N \times K$ matrices using the same procedure $\cG(r_1, r_2, K, N)$, where $r_1, r_2 \in \R_+^K$:
\begin{enumerate}
\item Define the procedure $\cM(r)$ generating symmetric matrices $K \times K$ for a given vector $r \in \R_+^K$:
  \begin{enumerate}
    \item[(i)] Simulate $K \times K$ matrix $S$
    \item[(ii)] Use SVD to decompose $S$ to $UDV'$
    \item[(iii)] Return $Q = V \diag[r] V' $, where $\diag[r]$ is a diagonal matrix with $r$ on the diagonal
  \end{enumerate}
\item Let $\Sigma_1 = \cM(r_1)$, $\Sigma_2 = \cM(r_2)$ and $\mu = \cN(0, I_K)$. 
\item Draw $N$ vectors $v_i$ from the distribution
\begin{align*}
v_i \sim 
\begin{cases}
\cN(2\mu, \Sigma_1) & \text{ if } 1 \leq i \leq N/3\\
\cN(-\mu, \Sigma_2) & \text{ if } N/3 < i \leq N
\end{cases}
\end{align*}
\item Return a matrix with rows $[v_i]_{1 \leq i \leq N}$.
\end{enumerate}
Define \[ r_1 = [1, 0.4, 0.005, 0.1 \exp(-3), ..., 0.1 \exp(-K+1)] \]  and \[ r_2 = [1.3, 0.2, 0.005, 0.1 \exp(-3), ..., 0.1 \exp(-K+1)],\]
and let $X_1,X_2,Z$ be generated following the procedure $\cG(r_1,r_2,K,N)$ and let $Y = Z + X_1 + X_2$. We consider $X_1,X_2$ and $Y$ as coefficients in a spline space $B$. We derive corresponding functions by multiplying these matrices by $B'$, i.e. $X_{f,1} = X_1B'$, $X_{f,2} = X_2B'$ and $Y_f = YB'$. We set $N=100$.

We choose $K=7$ and $T=31$ and we uniformly sample $10\%$ indices $\Omega \subset \{1,...,N\} \times \{1,...,T\}$, i.e. around $3$ points per curve on average. Each observed element of each matrix $X_{f,1}, X_{f,2}$ and $Y$ is drawn with Gaussian noise with mean $0$ and standard deviation $0.25$. The task is to recover $Y$ from sparsely observed elements $\{Y_{i,j} : (i,j) \in \Omega\}$. 

\subsection{Methods}

We compare \textsc{Soft-Longitudinal-Impute} (SLI) defined in Algorithm \ref{alg:soft-impute}, \textsc{Proximal-gradient} (PG) method defined in Algorithm \ref{alg:grid-free}, and the fPCA procedure \citep{james2000principal}, implemented by \citet{peng2009geometric}. All three algorithms require a set of basis functions. In all cases, we use the same basis $B$ as in the data generation process. In SLI and PG we also need to specify the tuning parameter $\lambda$, while in fPCA we need to choose the rank $R$. We use cross-validation to choose $\lambda$ and $R$ by optimizing for the prediction error on held-out (validation) observations.

We divide the observed coefficients into training ($81\%$), validation ($9\%$) and test ($10\%$) sets. We choose the best parameters of the three models on the validation set and then retrain on the entire training and validation sets combined. 
We compute the error by taking mean squared Frobenius distance between $Y$ and estimated $\hat{Y}$, i.e.
\begin{align}\label{eq:err}
 MSE(\hat{Y}) = \frac{1}{T|S|} \sum_{i\in S} \|Y_i - \hat{Y}_i \|_F^2
\end{align}
 on the test set $S$.
 

 We train the algorithms with all combinations of parameters: 
 regularization parameter for SLI and PG procedures $\lambda \in \{10, 15, 20, ..., 50\}$ and the rank for fPCA procedure $d \in \{2,3,4\}$. We compare the three methods fPCA, PG, and SLI, to the baseline {\it null model} which we define as the population mean across all visits.

\subsection{Results}

\begin{figure}[ht!]
  \includegraphics[width=0.49\linewidth]{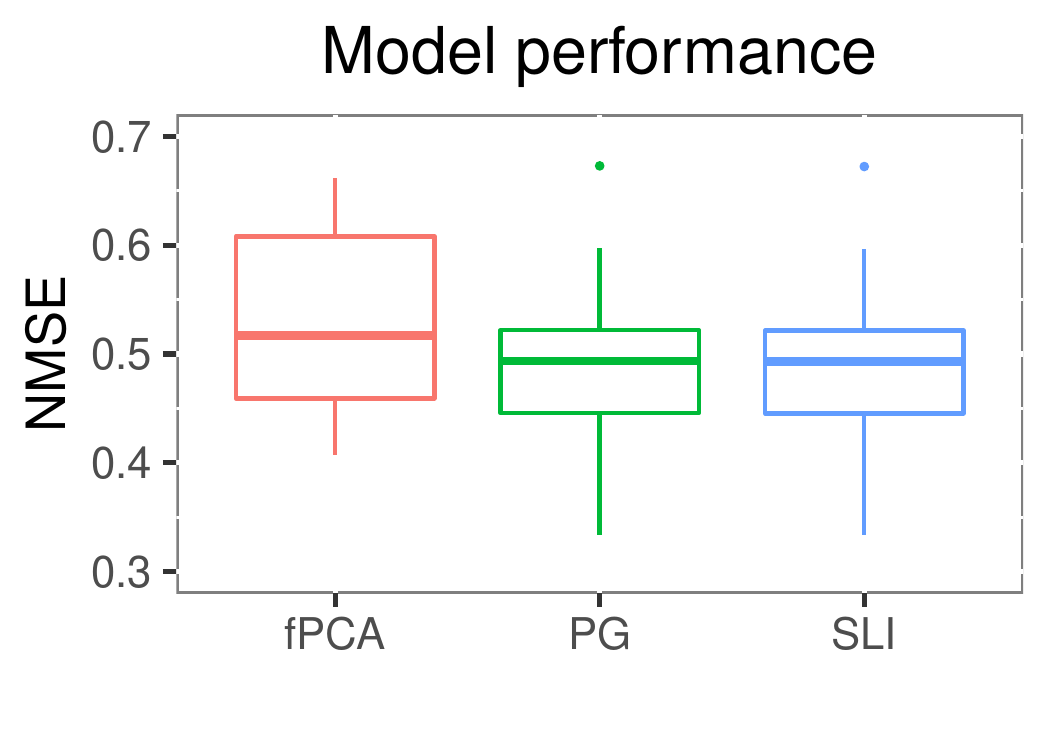}
  \includegraphics[width=0.49\linewidth]{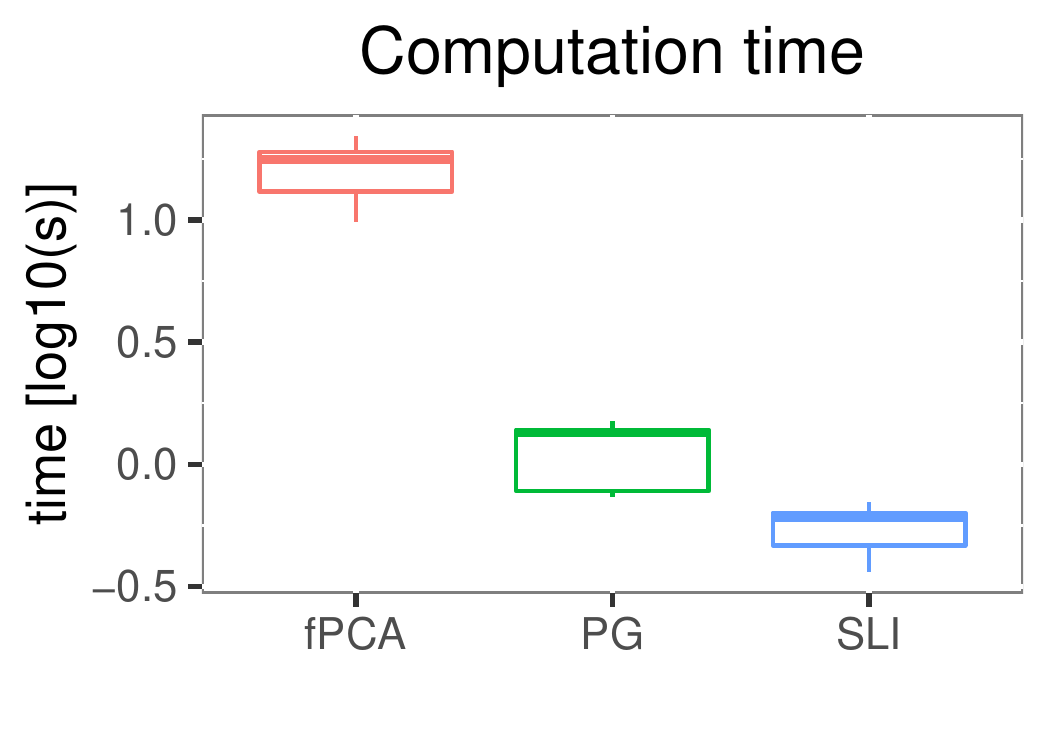}\\
  \centering\includegraphics[width=0.6\linewidth]{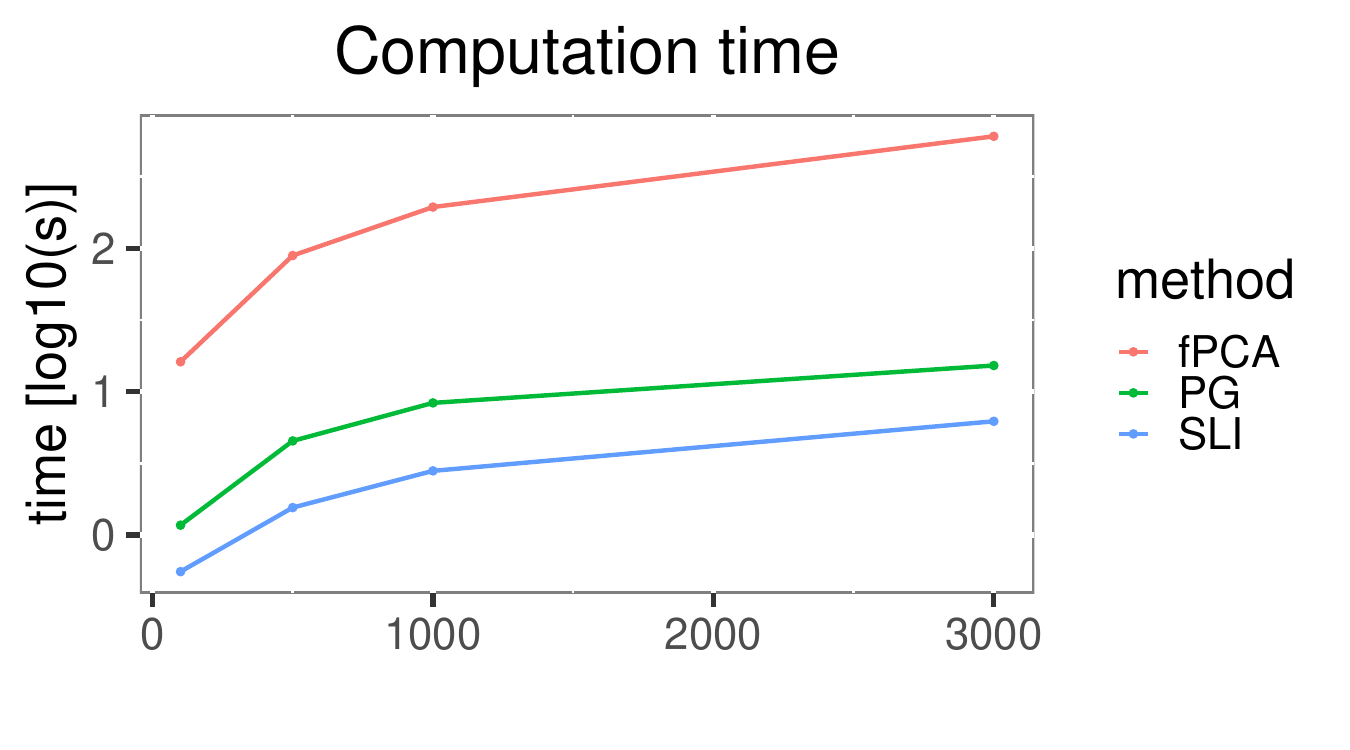}\\

  \caption{We illustrate performance (top left) and computation time (top right and bottom) of three estimation methods on simulated data. We compare functional PCA (fPCA; James et al. 2000) with Soft-Longitudinal-Impute (SLI) and proximal gradient (PG). Performance is measured by normalized mean squared error (NMSE; the lower the better). Boxplots in top panels represent distributions over $100$ repetitions with $n=100$ simulated subjects. In the bottom panel we illustrate computation time with different $n\in\{100,500,1000,3000\}$. Computation time is measured in seconds and represented on the log scale. On the absolute scale, with $n=3000$, methods fPCA, PG and SLI take on average $606, 15,$ and $6$ seconds.}
  \label{fig:performance}
\end{figure}

The SLI and PG methods achieve performance similar to fPCA \citep{james2000principal} but they are substantially faster as presented in Figure \ref{fig:performance}. In Figure \ref{fig:principal-components} we present the first components derived from both fPCA and SLI. In Figure \ref{fig:estimated-rank}, we present the estimated rank and cross-validation error of one of the simulation runs.

 \begin{figure}[ht!]
  \includegraphics[width=0.49\linewidth]{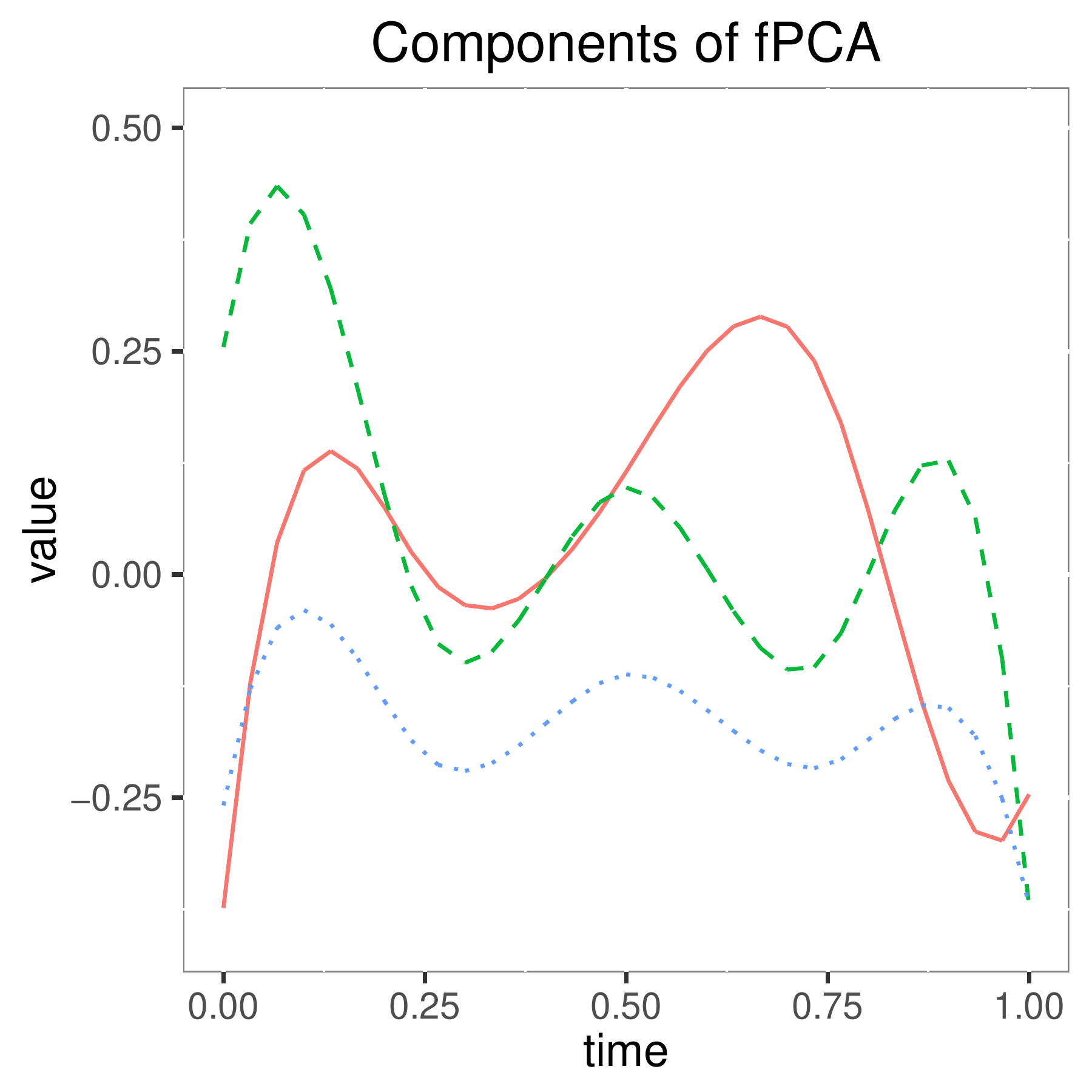}
  \includegraphics[width=0.49\linewidth]{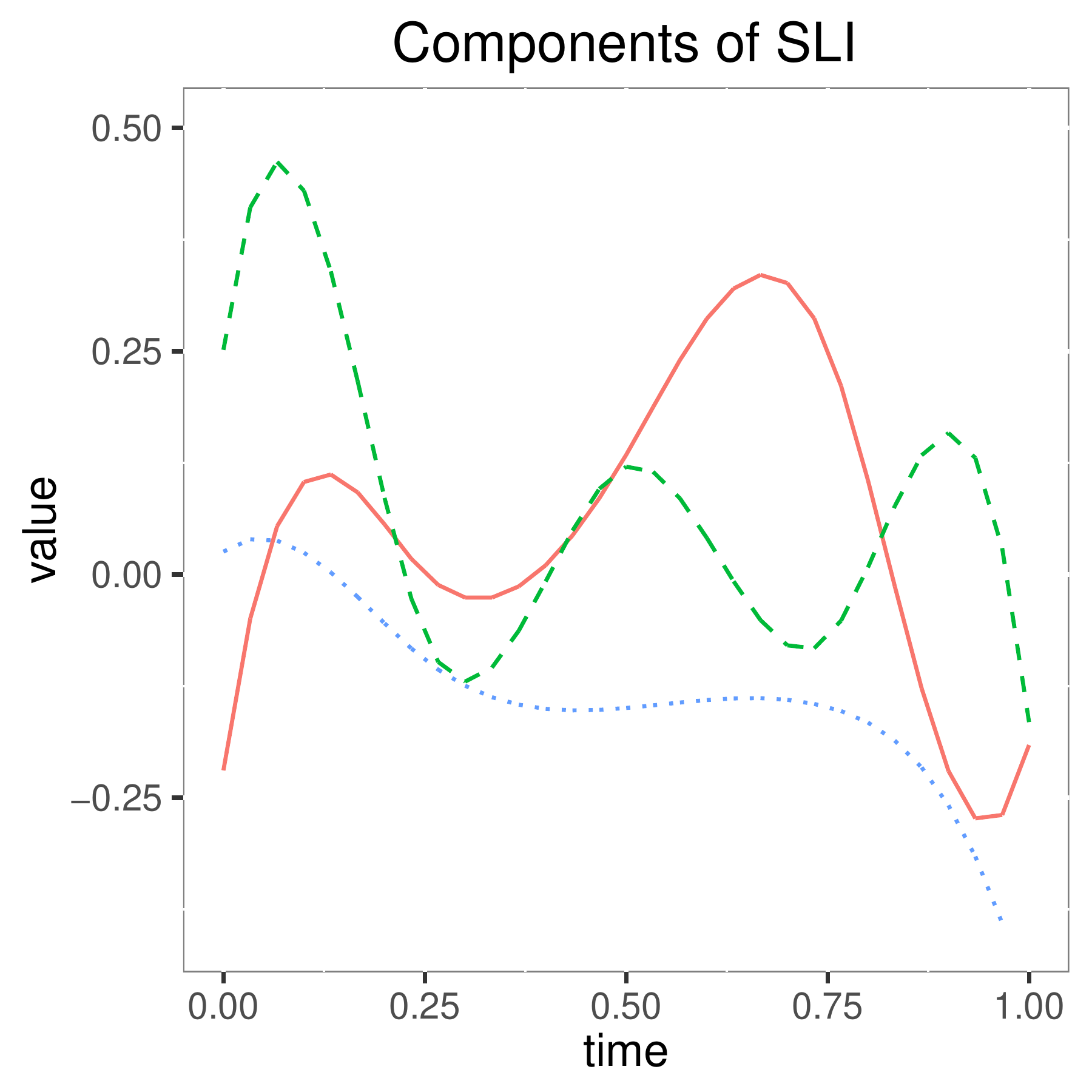}
  \caption{The first three principal components derived using sparse functional PCA (left) and Soft-Longitudinal-Impute (right) on the simulated data. The components are ordered as follows: 1st red solid curve, 2nd green dashed curve, 3rd blue dotted line. As expected, estimates of components are very similar.}
  \label{fig:principal-components}
\end{figure}


\begin{figure}[ht!]
  \includegraphics[width=0.49\linewidth]{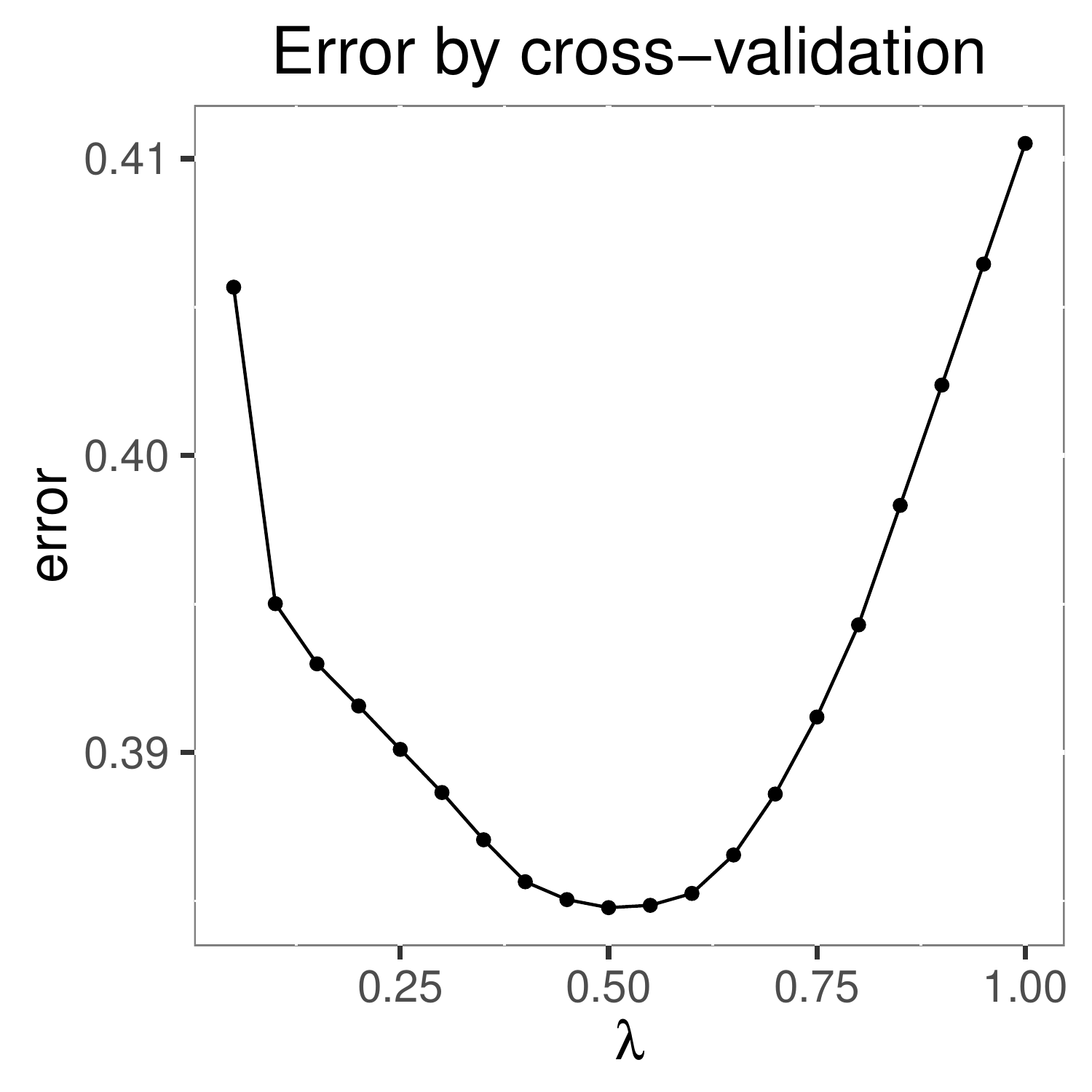}
  \includegraphics[width=0.49\linewidth]{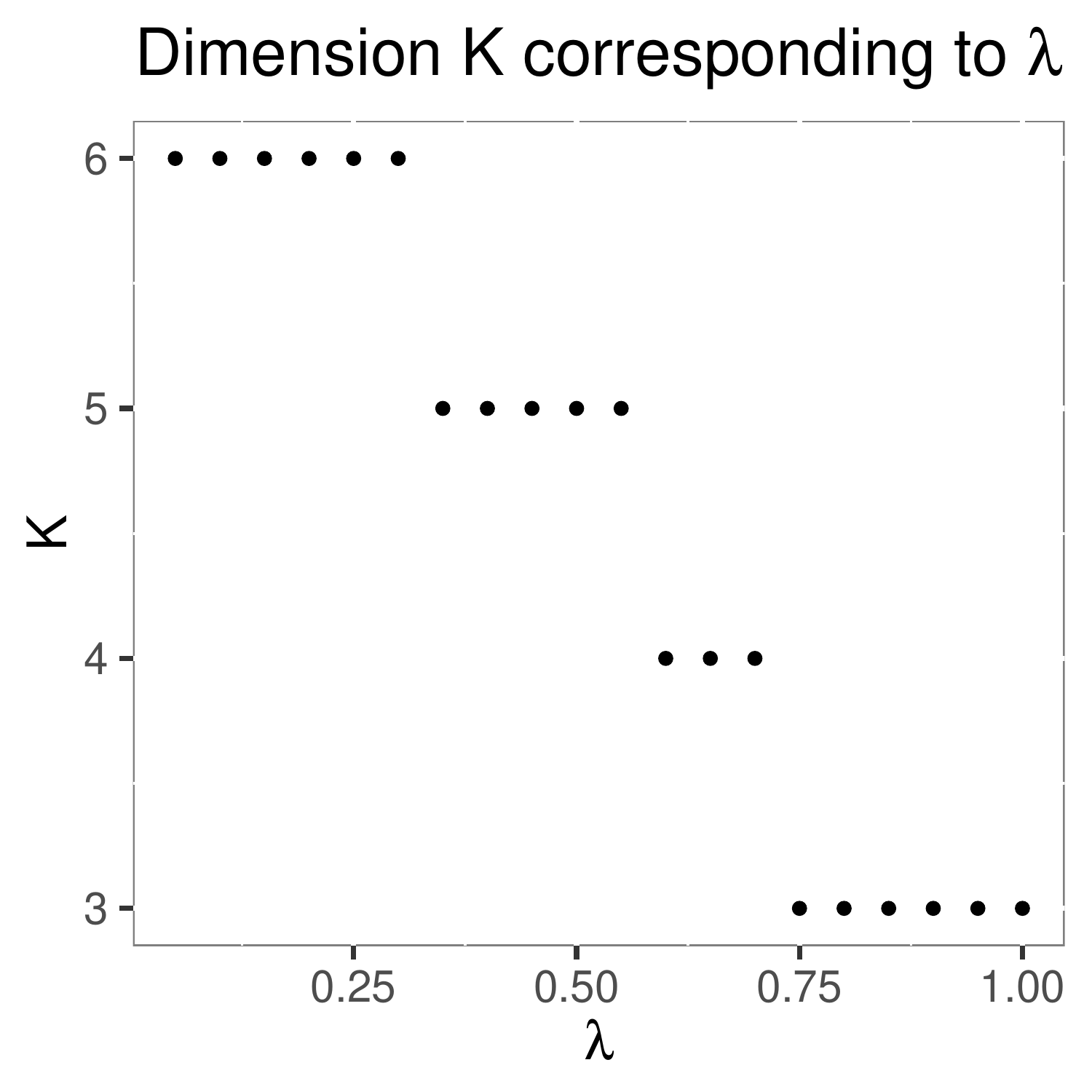}
  \caption{Our regularized procedure requires choice of the tuning parameter $\lambda$. This can be done using cross-validation where we optimize the mean-squared error of predictions on the held-out dataset. We plot the estimated error of the solution as a function of $\lambda$ in one of the simulations (left) and the estimated rank of the solution depending on the parameter $\lambda$ (right).}
  \label{fig:estimated-rank}
\end{figure}


\section{Data study}\label{s:data-study}

We present an application of our methods for understanding how longitudinal changes of gait patterns relate to subtypes of neurological disorders. First, we discuss how practitioners collect the data and use them to guide the decision process. Next, we describe our dataset and present how our methodology can improve current workflows.

In clinical gait analysis, at each visit movement of a child is recorded using optical motion capture. Optical motion capture allows estimating 3D positions of body parts using a set of cameras tracking markers positions on the subject's body. A set of at least three markers is placed at each analyzed body segment so that its 3D position and orientation can be identified uniquely. These data are then used to determine relative positions of body segments by computing the angle between the corresponding planes. Typically it is done using a biomechanical model for enforcing biomechanical constraints and improving accuracy.

In gait analysis practitioners usually analyze movement pattern of seven joints in lower limbs: ankle, knee, hip in each leg, and pelvis (Figure \ref{fig:joint-angles}). Each joint angle is measured in time. For making the curves comparable between the patients, usually, the time dimension is normalized to the percentage of the gait cycle, defined as the time between two foot strikes (Figure \ref{fig:joint-angles-in-time}).

While trajectories of joint angles are a piece of data commonly used by practitioners for taking decisions regarding treatment, their high-dimensional nature hinders their use as a quantitative metric of gait pathology or treatment outcome. This motivates development of univariate summary metrics of gait impairment, such as questionnaire-based metrics Gillette Functional Assessment Walking Scale (FAQ) \citep{gorton2011gillette}, Gross Motor Function Classification System (GMFCS) \citep{palisano2008content} and Functional Mobility Scale (FMS) \citep{graham2004functional}, or observational video analysis scores such as Edinburgh Gait Score \citep{read2003edinburgh}.

One of the most widely adopted quantitative measurements of gait impairments in pediatrics is gait deviation index (GDI) \citep{schwartz2008gait}. GDI is derived from joint angle trajectories and measures deviation of the first ten singular values from the population average of the typically developing population. GDI is normalized in such a way that $100$ corresponds to the mean value of typically developing children, with the standard deviation equal $10$. It is observed to be highly correlated with questionnaire-based methods. Thanks to its deterministic derivation from the motion capture measurements this method is considered more objective than questionnaires.

\begin{figure}[p]
\centering
  \includegraphics[width=0.83\linewidth]{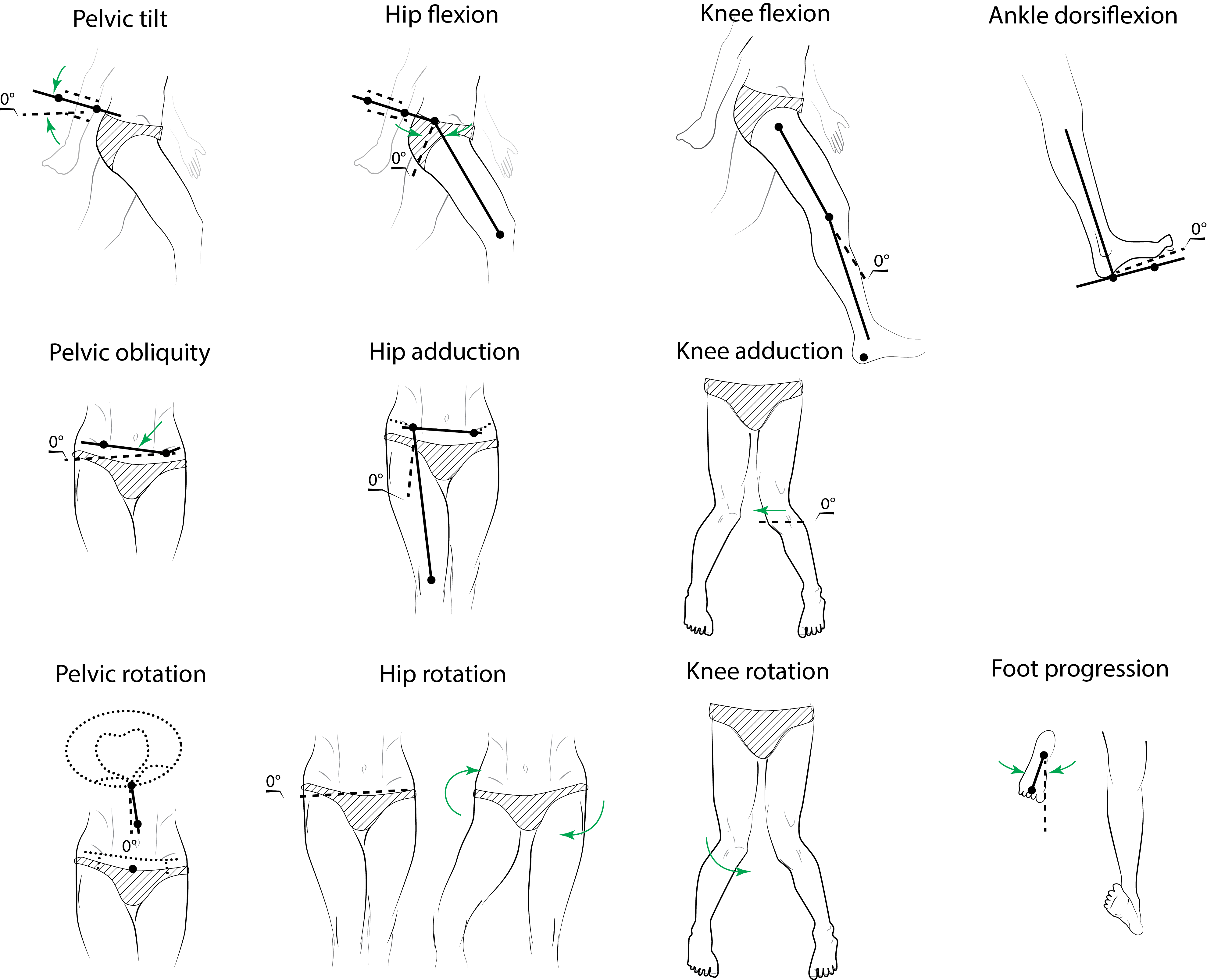}
  \caption{Four joints measured in clinical gait analysis: pelvis, hip, knee, and ankle. Each joint can be measured in three planes: sagittal plane (top row), frontal plate (middle row), and transverse plane (bottom row).}
    \label{fig:joint-angles}
  \includegraphics[width=0.83\linewidth]{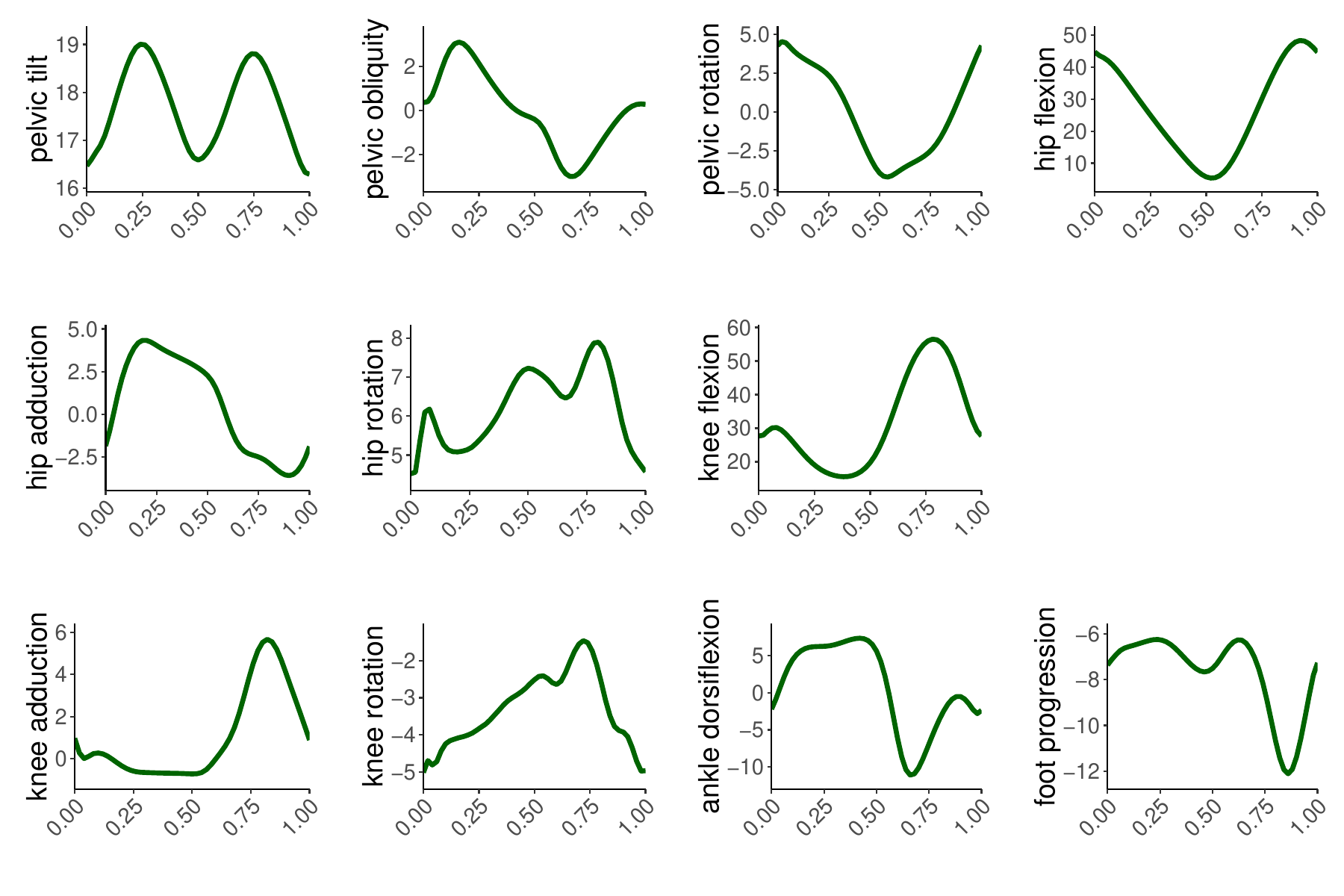}
  \caption{Example recordings of joint angles during the gait cycle of a single subject in one of the trials in our dataset. Each horizontal axis represents a fraction of the gait cycle.}
    \label{fig:joint-angles-in-time}
\end{figure}


In medical practice, GDI has been adapted as a metric for diagnosing the severity of impairment, and it constitutes an integral part of the clinical decision making process and evaluation of treatment outcomes. However, in order to correctly identify the surgery outcome, it is crucial to understand the natural progression of GDI. In particular, a positive outcome of a surgery might be negligible when compared to natural improvement during puberty. Similarly, a surgery maintaining the same level of GDI might be incorrectly classified as a failure, if the decline in patient's function over time is not accounted for.


Methods introduced in this article can be used to approximate individual progressions of GDI. First, we present how a prediction can be made solely based on the patient's GDI history and histories of other patients. Next, using our regression procedure, we predict GDI trajectories using other sparsely observed covariates, namely O$_2$ expenditure and walking speed.

\subsection{Materials and methods}

We analyze a dataset of Gillette Children's Hospital patients visiting the clinic between 1994 and 2014, age ranging between 4 and 19 years, mostly diagnosed with Cerebral Palsy. The dataset contains $84$ visits of $36$ patients without gait disorders and $6066$ visits of $2898$ patients with gait pathologies.

Motion capture data was collected at 120Hz and joint angles in time were extracted. These joint angles were then normalized in time to the gait cycle, resulting in curves as in Figure \ref{fig:joint-angles-in-time}. Points from these curves were then subsampled (51 equidistributed points) \tr{for the downstream analysis established in the clinic}.

In the dataset which we received from the hospital, for each patient we know their birthday and disease subtype. From each visit, we observe the following variables: patient ID, time of the visit, GDI of the left leg, GDI of the right leg, walking speed, and O$_2$ expenditure. Other clinical variables that we received were not included in this study. Walking speed is related to information we lose during normalization of the gait cycle in time. O$_2$ expenditure is a measure of a subject's energy expenditure during walking. Pathological gait is often energy inefficient and reduction of O$_2$ expenditure is one of the objectives of treatments. Finally, GDI is computed for two legs while in many cases the neurological disorder affects only one limb. To simplify the analysis, we focus on the more impaired limb by analyzing the minimum of the left and the right GDI.

Our objective is to model individual progression curves. We test two methods: functional principal components (fPCA) and \textsc{Soft-Longitudinal-Impute} (SLI). We compare the results to the \emph{null model} -- the population mean across all visits (\verb|mean|). In SLR, we approximate GDI using latent variables of sparsely observed covariates \textit{O$_2$ expenditure} and \textit{walking speed}, following the methodology from Section \ref{ss:regression}.

Let us denote the test set as $\Omega \subset \{1,2,...,N\} \times \{1,2,...,T\}$. We validate each model $M$ on held-out indices by computing the mean squared error as defined in \eqref{eq:err}. We select the parameters of each of the three methods using cross-validation, using the same validation set.

In our evaluation procedure, for the test set, we randomly select $5\%$ of observations of patients who visited the clinic at least $4$ times. Then, we split the remaining $95\%$ of observations into a training and validation sets in $90:10$ proportion. We train the algorithms with the following combinations of parameters: the regularization parameter for SLI and SLR procedures $\lambda \in \{0, 0.1, 0.2, ..., 2.0\}$ and the rank for fPCA procedure $d \in \{2,3,4,...,K\}$. We define the grid of $T = 51$ points. We repeat the entire evaluation procedure $20$ times.

\begin{table}[ht]
  \centering
\begin{tabular}{rrr}
  \hline
 & mean & sd \\ 
  \hline
  fPCA & 76.4 & 11.9\\
  SLI & 74.7 & 7.5\\
   \hline
\end{tabular}
\caption{Distribution of cross-validated MSE of the two methods: functional principal components (fPCA) and Soft-Longitudinal-Impute (SLI).} 
\label{tbl:data-res}
\end{table}

\subsection{Results}\label{ss:results}

Compared to the null model, fPCA and SLI explain around $\sim 30\%$ of the variance. We present detailed results in Table \ref{tbl:data-res}. SLR using additional predictors, O$_2$ expenditure and walking speed, yielded mean MSE of $0.68$ with standard devaition $0.8$. We conclude that O$_2$ expenditure and walking speed provide additional information for prediction of GDI progression.

\begin{figure}[ht!]
  \includegraphics[width=0.9\linewidth]{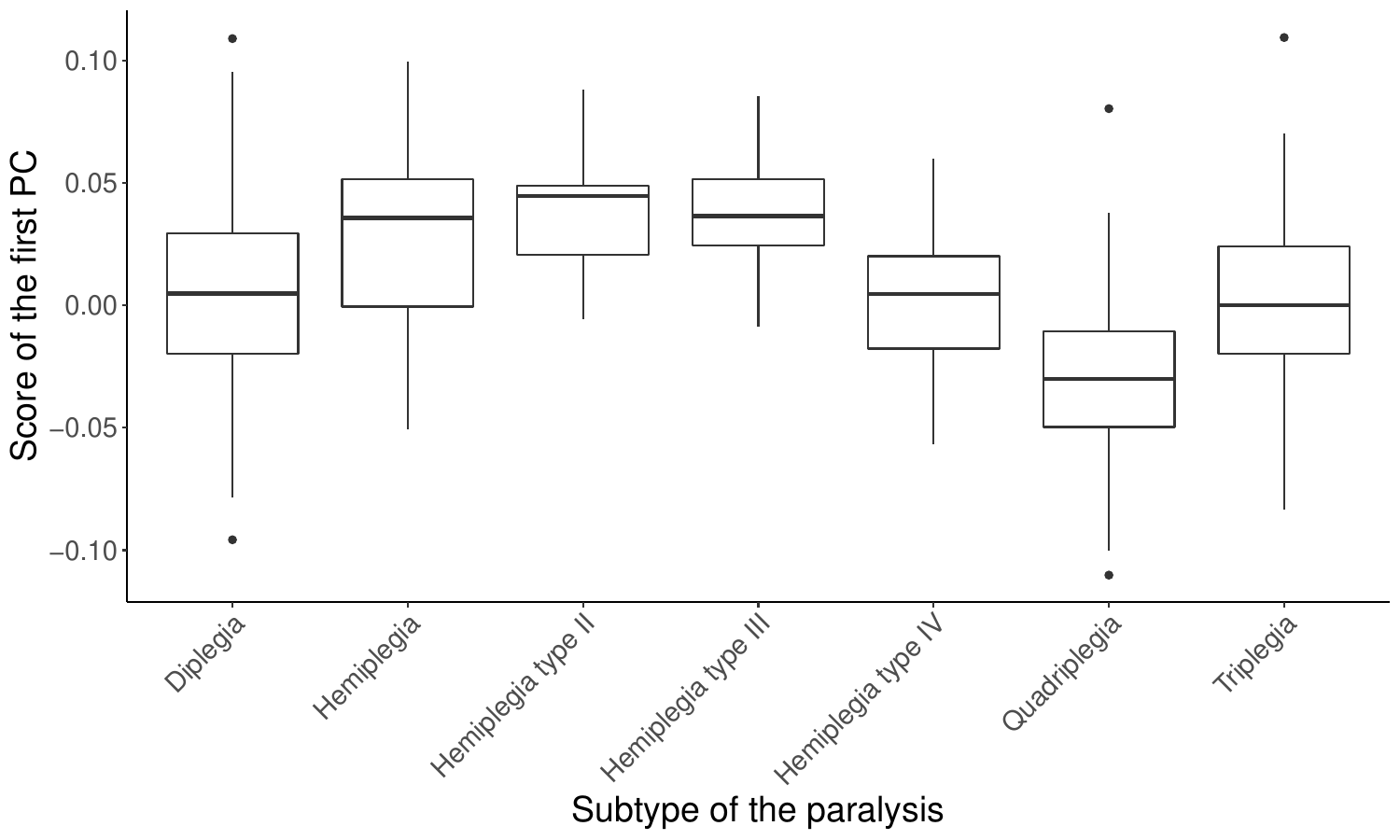}
\caption{Progression trends in different subsets of diseases. Negative values of the score, such as most of the quadriplegic group, correspond to individual trends where the first component (the red curve Figure \ref{fig:data-components} left) is subtracted from the mean (the green curve in Figure \ref{fig:data-components} right). Positive values of the score, such as most of the hemiplegic group, correspond to individual trends where the first component is added (the red curve in Figure \ref{fig:data-components} right).}
    \label{fig:subtypes}
\end{figure}

Both fPCA and \textsc{Sparse-Longitudinal-Impute} provide latent representations of patients' progression curves. We analyze the singular value vectors from our SVD solution which we refer to as principal components. In the left plot in Figure \ref{fig:data-components} we show the first two estimated principal components. We found that the first component estimates the change between GDI before and after age of 20. The second component models changes around age of 10 and around age of 18. In the right \tr{plot} in Figure \ref{fig:data-components}, by adding a principle component to the population mean curve, we illustrate how differences in the first component are reflected in the patients trajectory. By visual investigation of curves returned by our \textsc{Sparse-Longitudinal-Impute} and by fPCA we found similar trends in the first two components.

Since our SVD decomposition defines a low-rank representation of the progression trends, we can also use it to gain insights on progression in different groups of patients. In Cerebral Palsy we divide paralysis into subtypes depending on which limbs are affected: monolegia (one leg), diplegia (two legs), hemiplegia (one side of the body), triplegia (three limbs), quadriplegia (four limbs). Hemiplegia is the most prevalent in our population and it might be divided depending on severity, from type I (weak muscles, drop foot) to type IV (severe spasticity). We find differences between trends of progression for different subtypes of paralysis of patients ($F_{6,541} = 17.17, p < 10^{-15}$). We illustrate these distributions in Figure \ref{fig:subtypes}.

\section{Discussion}\label{s:discussion}

Results presented in Section \ref{ss:results} imply that our \textsc{Sparse-Longitudinal-Impute} and \textsc{Sparse-Longitudinal-Regression} methods can be successfully applied to understands trends of variability of disease progression. We show how to incorporate progressions of O$_2$ expenditure and walking speed in the prediction of the progression of GDI. We present how low-rank representation can be leveraged to gain insights about subtypes of impairment.

While a large portion of variance remains unexplained, it is important to note that in practice the individual progression is not accounted for explicitly in the current decision-making process. Instead, practitioners only use the population-level characteristics of the dependence between age and impairment severity. Our model can greatly improve this practice.

Despite successful application, we identify limitations that could be potentially addressed in the extensions of our model. First, the method is meant to capture natural continuous progression of GDI, while in practice there are many discrete events, such as surgeries that break continuity assumption and render the mean trajectories less interpretable. Second, our methodology does not address the ``cold start problem'', i.e. we do not provide tools for predictions with only one or zero observations. Third, we do not provide explicit equations for confidence bounds of predicted parameters.

While these and other limitations can constrain applicability of the method in the current form, they can be addressed using existing techniques of matrix completion. The focus of this paper is to introduce a computational framework rather than build a full solution for all cases. Elementary formulation of the optimization problem as well as the fully-functional \verb|R| implementation can foster development of new tools using matrix completion for longitudinal analysis and for mixed-effect models.

In our \verb|R| package \verb|fcomplete| available at \url{https://github.com/kidzik/fcomplete}, we provide implementations of all algorithms described in this article as well as helper functions for transforming the data, sampling training and test datasets, and plotting functions. For convenience, we also provided an interface for using the \verb|fpca| package implementing Sparse Functional Principal Components algorithms \citep{james2000principal,peng2009geometric}. The analysis was perform on a desktop PC with 64 GB RAM memmory and an Intel\textsuperscript{\textregistered} Core\textsuperscript{\texttrademark} Intel(R) Core(TM) i9-10900K CPU @ 3.70GHz, operating on a Ubuntu 18.04 system with \verb|R| version 4.1.0.

\section{Acknowledgements}
Łukasz Kidziński was supported by the Mobilize Center grant U54 EB020405 from the National Institute of Health. Trevor J. Hastie was partially supported by grants DMS-2013736 And IIS, 1837931 from the National Science Foundation, and grant 5R01 EB, 001988-21 from the National Institutes of Health.



\renewcommand*{\bibfont}{\small}
\bibliography{report}

\appendix

\section{Proofs}\label{s:convergence}

We prove convergence by mapping our problem into the framework of \citet{mazumder2010spectral}.
Following their notation we define
\begin{align}\label{eq:helper-function}
 f_\lambda(W) = \frac{1}{2} \|P_{\Omega}(Y) - P_{\Omega}(WB')\|_F^2 + \lambda\|W\|_*.
\end{align}
Our objective is to find $ W_\lambda = {\arg\min}_W f_\lambda(W)$. We define
\begin{align}\label{eq:q}
Q_\lambda(W|\tilde{W}) = \frac{1}{2}\|P_{\Omega}(Y) + P_{\Omega}^{\perp}(\tilde{W} B') - WB'\|_F^2 + \lambda \|W\|_*.
\end{align}
Algorithm \ref{alg:soft-impute} in the step $k$ computes $W_\lambda^{k+1} = \argmin_W Q_\lambda(W|W_\lambda^k)$. We show that $W_\lambda^k$ converges to the solution of \eqref{eq:helper-function}.
\begin{lemma}\label{lemma:svt}
Let $W$ be an $N \times K$ matrix of rank $r \leq K$ and $B$ is an orthogonal $T \times K$ matrix. The solution to the optimization problem
\begin{equation}\label{eq:lemma1}
\min_W \frac{1}{2}\|Y - WB' \|_F^2 + \lambda\|W\|_*
\end{equation}
is given by $\hat{W} = S_\lambda(YB)$ where
\[
S_\lambda(YB) \equiv WD_\lambda V' \text{ with } D_\lambda = \diag[(d_1 - \lambda)_+, ..., (d_r - \lambda)_+],
\]
$WDV'$ is the SVD of $YB$, $D = \diag[d_1,...,d_r]$, and $t_+ = \max(t,0)$.
\end{lemma}
\begin{proof}
By Lemma 1 from \citet{mazumder2010spectral} we know that $S_\lambda(YB)$ solves
\[
\min_W \frac{1}{2}\|YB - W \|_F^2 + \lambda\|W\|_*.
\]
Now, since we have $\|Y - WB' \|_F = \|YB - W \|_F$, $S_\lambda(YB)$ also solves our Lemma~\ref{lemma:svt}.
\end{proof}
\begin{lemma}\label{eq:z-sequence}
For every fixed $\lambda \geq 0$, define a sequence $W_\lambda^k$ by
\[
W_\lambda^{k+1} = \argmin_W Q_\lambda(W|W_\lambda^k)
\]
with any starting point $W_\lambda^0$. The sequence $W_\lambda^k$ satisfies 
\[
f_\lambda(W_\lambda^{k+1}) \leq  Q_\lambda(W_\lambda^{k+1} | W_\lambda^k ) \leq f_\lambda(W_\lambda^{k})
\]
\end{lemma}
\begin{proof}
By Lemma \ref{lemma:svt} and the definition \eqref{eq:q}, we have:
\begin{align*}
  f_\lambda(W_\lambda^k) &= Q_\lambda(W_\lambda^k | W_\lambda^k)\\
  &= \frac{1}{2}\|P_{\Omega}(Y) + P_{\Omega}^{\perp}(W_\lambda^k B') - W_\lambda^kB'\|_F^2 + \lambda \|W_\lambda^k\|_*\\
  &\geq \min_W \frac{1}{2}\|P_{\Omega}(Y) + P_{\Omega}^{\perp}(W_\lambda^k B') - WB'\|_F^2 + \lambda \|W\|_*\\
  &= Q_\lambda(W_\lambda^{k+1} | W_\lambda^k)\\
  &= \frac{1}{2}\|P_{\Omega}(Y) + P_{\Omega}^{\perp}(W_\lambda^k B') - W_\lambda^{k+1}B'\|_F^2 + \lambda \|W_\lambda^{k+1}\|_*\\
  &= \frac{1}{2}\|(P_{\Omega}(Y) - P_\Omega(W_\lambda^{k+1}B'))+ (P_{\Omega}^{\perp}(W_\lambda^k B') - P_\Omega^\perp(W_\lambda^{k+1}B'))\|_F^2 + \lambda \|W_\lambda^{k+1}\|_*\\
  &= \frac{1}{2}\|P_{\Omega}(Y) - P_\Omega(W_\lambda^{k+1}B')\|_F^2 + \frac{1}{2}\|P_{\Omega}^{\perp}(W_\lambda^k B') - P_\Omega^\perp(W_\lambda^{k+1}B')\|_F^2 + \lambda \|W_\lambda^{k+1}\|_*\\
  &\geq \frac{1}{2}\|P_{\Omega}(Y) - P_\Omega(W_\lambda^{k+1}B')\|_F^2 + \lambda \|W_\lambda^{k+1}\|_*\\
  &= Q_\lambda(W_\lambda^{k+1} | W_\lambda^{k+1})\\
  &= f(W_\lambda^{k+1}).
\end{align*}
\end{proof}
    Note that proofs of Lemma \ref{lemma:svt} and Lemma \ref{eq:z-sequence} are just extentions of their counterparts in \citet{mazumder2010spectral} with the basis $B$ included. Similarly, we get equivalent results for their Lemma 3-5 and the main theorem.
\begin{theorem}
The sequence of $W_{\lambda}^k$ defined in Lemma \ref{eq:z-sequence} converges to a limit $W_\lambda^\infty$ that solves
\[
\min_W \frac{1}{2} \|P_\Omega(Y) - P_\Omega(WB)\|_F^2 + \lambda\|W\|_*.
\]
\end{theorem}




\end{document}